\documentclass[letterpaper]{article} % DO NOT CHANGE THIS
\usepackage{aaai25}  % DO NOT CHANGE THIS
\usepackage{times}  % DO NOT CHANGE THIS
\usepackage{helvet}  % DO NOT CHANGE THIS
\usepackage{courier}  % DO NOT CHANGE THIS
\usepackage[hyphens]{url}  % DO NOT CHANGE THIS
\usepackage{graphicx} % DO NOT CHANGE THIS
\usepackage{natbib}  % DO NOT CHANGE THIS AND DO NOT ADD ANY OPTIONS TO IT
\usepackage{caption} % DO NOT CHANGE THIS AND DO NOT ADD ANY OPTIONS TO IT
\usepackage{algorithm}
\usepackage{algorithmic}
\usepackage{subfigure}
\usepackage{diagbox}
\usepackage{booktabs}
\usepackage{amsmath}
\usepackage{amssymb}
\usepackage{mathtools}
\usepackage{amsthm}
\usepackage{tikz}
\usepackage{bm}
\usepackage{esvect}
\usepackage{multirow}

\theoremstyle{plain}
\newtheorem{theorem}{Theorem}

\theoremstyle{definition}

\theoremstyle{remark}

\usepackage{newfloat}
\usepackage{listings}
\DeclareCaptionStyle{ruled}{labelfont=normalfont,labelsep=colon,strut=off} % DO NOT CHANGE THIS
\floatstyle{ruled}
\newfloat{listing}{tb}{lst}{}
\floatname{listing}{Listing}
\title{A Variance Minimization Approach to Temporal-Difference Learning
\footnote{Coincidentally, the core idea of this paper is similar to
 reward centering \cite{naik2024rewardcentering}. We have been submitting it since last year, 
 but it has not been accepted yet, so we had to put it on arXiv.}}
\author {
    Xingguo Chen\textsuperscript{\rm 1},
    YuGong\textsuperscript{\rm 1},
    Shangdong Yang\textsuperscript{\rm 1},
    Wenhao Wang\textsuperscript{\rm 2}
}
\affiliations {
    \textsuperscript{\rm 1}Jiangsu Key Laboratory of Big Data Security and Intelligent Processing, Nanjing University of Posts and Telecommunications, Nanjing 210023, China\\
    \textsuperscript{\rm 2}College of Electronic Engineering, National University of Defense Technology, China
}

\usepackage{bibentry}
\begin{document}
\setcounter{theorem}{0}
\maketitle
\begin{abstract}
    Fast-converging algorithms are a contemporary requirement 
    in reinforcement learning. In the context of linear function 
    approximation, the magnitude of the smallest 
    eigenvalue of the key matrix is a major factor 
    reflecting the convergence speed. Traditional value-based 
    RL algorithms focus on minimizing errors.
    This paper introduces a variance minimization (VM) approach for value-based RL instead of error minimization.
    Based on this approach, we proposed two objectives, the Variance of Bellman Error (VBE) and the Variance of Projected Bellman Error (VPBE), and derived the VMTD, VMTDC, and VMETD algorithms.
    We provided proofs of their convergence and optimal policy invariance of the variance minimization.
    Experimental studies validate the effectiveness of the proposed algorithms.
    % Therefore, this paper proposes two new objective 
    % functions and derives three Variance Minimization (VM) algorithms, including VMTD, VMTDC, and VMETD.
    % A scalar parameter, $\omega$, is introduced, to improve the performance of parametric 
    % temporal difference learning algorithms.

    % In the policy evaluation experiment, two-state, 
    % we analyze 
    % the convergence speed of these algorithms by calculating the minimum eigenvalue of the key 
    % matrices both on-policy and off-policy.In controlled experiments, the VM algorithms demonstrate 
    % superior performance.
\end{abstract}


\section{Introduction}
\label{introduction}
Reinforcement learning (RL) can be mainly divided into two
categories: value-based reinforcement learning
and policy gradient-based reinforcement learning. This
paper focuses on temporal difference learning based on
linear approximated valued functions. Its research is
usually divided into two steps: the first step is to establish the convergence of the algorithm, and the second
step is to accelerate the algorithm.

In terms of stability, \citet{sutton1988learning} established the
 convergence of on-policy TD(0), and \citet{tsitsiklis1997analysis}
 established the convergence of on-policy TD($\lambda$).
 However, ``The deadly triad'' consisting of off-policy learning, 
 bootstrapping and function approximation makes 
 the stability  a difficult problem \citep{Sutton2018book}.
 To solve this problem, convergent off-policy temporal difference
 learning algorithms are proposed, e.g., BR \cite{baird1995residual},
 GTD \cite{sutton2008convergent},  GTD2 and TDC \cite{sutton2009fast},
 ETD \cite{sutton2016emphatic}, and MRetrace \cite{chen2023modified}.

In terms of acceleration, \citet{hackman2012faster} 
proposed a Hybrid TD algorithm with the on-policy matrix.
\citet{liu2015finite,liu2016proximal,liu2018proximal} proposed
true stochastic algorithms, i.e., GTD-MP and GTD2-MP, from 
a convex-concave saddle-point formulation.
Second-order  methods are used to accelerate TD learning,
e.g.,  Quasi Newton TD \cite{givchi2015quasi} and 
accelerated TD (ATD)  \citep{pan2017accelerated}.
\citet{hallak2016generalized} introduced a new parameter 
to reduce variance for ETD.
\citet{zhang2022truncated} proposed truncated ETD with a lower variance.
Variance Reduced TD with direct variance reduction technique \citep{johnson2013accelerating} is proposed by \cite{korda2015td} 
and analysed by  \cite{xu2019reanalysis}.
How to further improve the convergence rates of reinforcement learning 
algorithms is currently still an open problem.

Algorithm stability is prominently reflected in the changes 
to the objective function, transitioning from mean squared 
errors (MSE) \citep{Sutton2018book} to mean squared bellman errors (MSBE) \cite{baird1995residual}, then to 
norm of the expected TD update \cite{sutton2009fast}, and further to 
mean squared projected Bellman errors (MSPBE) \cite{sutton2009fast}. On the other hand, the algorithm 
acceleration is more centered around optimizing the iterative 
update the formula of the algorithm itself without altering the 
the objective function, thereby speeding up the convergence rate 
of the algorithm. The emergence of new optimization objective 
functions often lead to the development of novel algorithms. 
The introduction of new algorithms, in turn, tends to inspire 
researchers to explore methods for accelerating algorithms, 
leading to the iterative creation of increasingly superior algorithms.

The kernel loss function can be optimized using standard 
gradient-based methods, addressing the issue of double 
sampling in residual gradient algorithm \cite{feng2019kernel}. It ensures convergence 
in both on-policy and off-policy scenarios. The logistic bellman 
error is convex and smooth in the action-value function parameters, 
with bounded gradients \cite{basserrano2021logistic}. In contrast, the squared Bellman error is 
not convex in the action-value function parameters, and RL algorithms 
based on recursive optimization using it are known to be unstable.

It is necessary to propose a new objective function, but the abovementioned objective functions are all some form of error.
Is minimizing error the only option for value-based reinforcement learning?

\textbf{The contributions of this paper} are as follows:
(1) Introduction of variance minimization (VM) approach for value-based RL instead of error minimization.
(2) Based on this approach, we proposed two objectives, the Variance of Bellman Error (VBE) and the Variance of Projected Bellman Error (VPBE), and derived the VMTD, VMTDC, and VMETD algorithms.
(3) We provided proofs of their convergence and optimal policy invariance.

\section{Background}
\label{preliminaries}
\subsection{Markov Decision Process}
Reinforcement learning agent interacts with the environment, observes the state,
 takes sequential decision-making to influence the environment, and obtains
 rewards.
 Consider an infinite-horizon discounted 
 Markov Decision Process (MDP), defined by a tuple $\langle S,A,R,P,\gamma
 \rangle$, where $S=\{1,2,\ldots,N\}$ is a finite set of states of the environment;  $A$
 is a finite set of actions of the agent; 
 $R:S\times A \times S \rightarrow \mathbb{R}$ is a bounded deterministic reward
 function; $P:S\times A\times S \rightarrow [0,1]$ is the transition
 probability distribution;  and $\gamma\in (0,1)$
 is the discount factor \cite{Sutton2018book}.
 Due to the requirements of  online learning, value iteration based on sampling
 is considered in this paper. 
 In each sampling, an experience (or transition) $\langle s, a, s', r\rangle$ is
 obtained.

 A policy is a mapping $\pi:S\times A \rightarrow [0,1]$. The goal of the
 agent is to find an optimal policy $\pi^*$ to maximize the expectation of a
 discounted cumulative rewards over a long period. For each discrete time step 
 $t=0,1,2,3,…$, 
 State value function $V^{\pi}(s)$ for a stationary policy $\pi$ is 
 defined as:  \begin{equation*}
 V^{\pi}(s)=\mathbb{E}_{\pi}[\sum_{k=0}^{\infty} \gamma^k R_{t+k+1}|S_t=s].
 \label{valuefunction}
 \end{equation*}
 Linear value function for state $s\in S$ is defined as:
\begin{equation}
 V_{{\theta}}(s):= {{\theta}}^{\top}{{\phi}}(s) = \sum_{i=1}^{m}
\theta_i \phi_i(s),
\label{linearvaluefunction}
\end{equation}
 where ${{\theta}}:=(\theta_1,\theta_2,\ldots,\theta_m)^{\top}\in
 \mathbb{R}^m$ is a parameter vector, 
 ${{\phi}}:=(\phi_1,\phi_2,\ldots,\phi_m)^{\top}\in \mathbb{R}^m$ is a feature
 function defined on state space $S$, and $m$ is the feature size. 

 Tabular temporal difference (TD) learning \cite{Sutton2018book} has been successfully applied to small-scale problems.
 To deal with the well-known curse of dimensionality of large-scale MDPs, the value
 function is usually approximated by a linear model (the focus of this paper), kernel methods, decision
 trees, neural networks, etc. 

\subsection{On-policy and Off-policy Learning}
\begin{table*}[t]
\caption{Minimum eigenvalues of $\frac{1}{2}({\textbf{A}+\textbf{A}^{\top}})$ for various algorithms in the 2-state.}
\label{tab:min_eigenvalues} % 添加标签
\vskip 0.15in
\begin{center}
\begin{small}
\begin{sc}
\begin{tabular}{l|cc|cc|cr} % 在这里加了竖线
\toprule
 algorithm &TD &VMTD &TDC &VMTDC & ETD&VMETD \\
\midrule
 ON-POLICY& $\bm{0.475}$&$0.25$  & $\bm{0.09025}$& $0.025$&$\bm{4.75}$ & $2.5$ \\
 OFF-POLICY&$-0.2$ & $\bm{0.25}$&$0.016$&$\bm{0.025}$ & $\bm{3.4}$ & $1.15$\\
\bottomrule
\end{tabular}
\end{sc}
\end{small}
\end{center}
\vskip -0.1in
\end{table*}
On-policy and off-policy algorithms are currently hot topics in research.  
The main difference between the two lies in the fact that in on-policy algorithms, 
the behavior policy $\mu$ and the target policy $\pi$ are the same during the learning process. 
In off-policy algorithms, however, the behavior policy and the target policy are different. 
The algorithm uses data generated from the behavior policy to optimize the 
target policy, which leads to higher sample efficiency and complex stability issues.

From the theory of stochastic methods, the 
the convergence point of linear TD algorithms is a parameter vector, say $\bm{\theta}$, that satisfies 
\begin{equation*}
\begin{array}{ccl}
 b - \textbf{A}{\theta}&=&0,
\end{array}
\end{equation*}
where $\textbf{A}\in \mathbb{R}^{|S| \times m}$ and $b\in \mathbb{R}^{m}$. 
If the matrix 
$\textbf{A}$ is positive definite, then the algorithm converges.

\begin{theorem}
\label{lemma_1}
(The main factor affecting convergence rates \cite{chen2024convergence}).
%  Assume Assumptions 1, 2, 3, 6 and 7, 
 Assume the same parameters setting for each algorithm, 
% it can be concluded that 
from the perspective of the expected convergence 
rate, the main factor that affects the convergence rate is the minimum eigenvalue 
of the matrix $\frac{1}{2}({\textbf{A}+\textbf{A}^{\top}})$.
 The larger the minimum eigenvalue, the faster the convergence rate.

\end{theorem}
% The proof of Theorem \ref{lemma_1}  and Assumptions 1, 2, 3, 6, and 7 can be found in Corollary 2 of \citet{chen2024convergence}.

Next, we will compute the minimum eigenvalue of 
$\textbf{A}$ for TD(0), TDC, and ETD in both on-policy and off-policy settings in a 2-state environment.
First, we will introduce the environment setup for the 2-state case in both on-policy and off-policy settings.

The ``1''$\rightarrow$``2'' problem has only two states. From each
state, there are two actions, left and right, which take
\begin{minipage}{0.2\textwidth}
the agent to the left or right state. All rewards are zero.
The feature $\bm{\Phi}=(1,2)^{\top}$ 
are assigned to the left and the right 
\end{minipage}
\begin{minipage}{0.25\textwidth}
  \centering
  \includegraphics[scale=0.25]{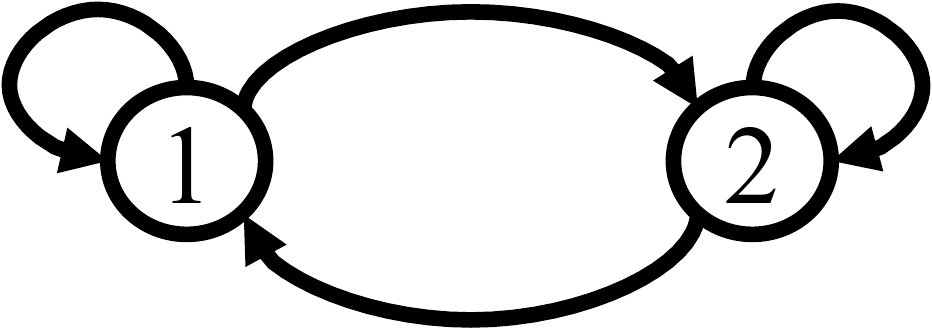}
\end{minipage}
state. The first policy takes equal
probability to left or right
in both states, i.e., 
$
\textbf{P}_{1}=
\begin{bmatrix}
0.5 & 0.5 \\
0.5 & 0.5
\end{bmatrix}
$.
The second policy only selects action rights in both states, i.e., 
$
\textbf{P}_{2}=
\begin{bmatrix}
0 & 1 \\
0 & 1
\end{bmatrix}
$.
The state distribution of
the first policy is $d_1 =(0.5,0.5)^{\top}$.
The state distribution of
the second policy is $d_1 =(0,1)^{\top}$.
The discount factor is $\gamma=0.9$.
In the on-policy setting, the behavior policy 
and the target policy are the same, so 
let $\textbf{P}_{\mu}=\textbf{P}_{\pi}=\textbf{P}_{1}$. 
In the off-policy setting, 
let $\textbf{P}_{\mu}=\textbf{P}_{1}$ and $\textbf{P}_{\pi}=\textbf{P}_{2}$. 

The key matrix $\textbf{A}_{\text{on}}$ of on-policy TD(0) is
\begin{equation*}
 \textbf{A}_{\text{on}} = \bm{\Phi}^{\top}\textbf{D}_{\pi}(\textbf{I}-\gamma \textbf{P}_{\pi})\bm{\Phi},
\end{equation*}
where $\bm{\Phi}$ is the $|S| \times m$ matrix with the $\phi(s)$ as its rows, and $\textbf{D}_{\pi}$ is the $|S| \times |S|$ diagonal
matrix with $d_{\pi}$ on its diagonal. $d_{\pi}$ is a vector, each component representing the steady-state
distribution under policy $\pi$. $\textbf{P}_{\pi}$ denote the $|S| \times |S|$ matrix of transition probabilities under $\pi$. And $\textbf{P}_{\pi}^{\top}d_{\pi}=d_{\pi}$.

The key matrix $\textbf{A}_{\text{off}}$ of off-policy TD(0) is
\begin{equation*}
 \textbf{A}_{\text{off}} = \bm{\Phi}^{\top}\textbf{D}_{\mu}(\textbf{I}-\gamma \textbf{P}_{\pi})\bm{\Phi},
\end{equation*}
where $\textbf{D}_{\mu}$ is the $|S| \times |S|$ diagonal
matrix with $d_{\mu}$ on its diagonal. $d_{\mu}$ is a vector, each component representing the steady-state
distribution under behavior policy $\mu$.

In the off-policy 2-state, 
$\textbf{A}_{\text{off}}=-0.2$, which means that off-policy TD(0) cannot stably converge, 
while , in the on-policy 2-state, $\textbf{A}_{\text{on}}=0.475$, which means that on-policy TD(0) can stably converge.

The key matrix $\textbf{A}_{\text{TDC}}= \textbf{A}^{\top}_{\text{off}}\textbf{C}^{-1}\textbf{A}_{\text{off}}$, 
where $\textbf{C}=\mathbb{E}[\bm{\bm{\phi}}\bm{\bm{\phi}}^{\top}]$. 
In the 2-state counterexample, 
$\textbf{A}_{\text{TDC}}=0.016$, which means that TDC can stably converge.

The key matrix $\textbf{A}_{\text{TDC}}$ of on-policy TDC is
\begin{equation*}
 \textbf{A}_{\text{TDC}} = \textbf{A}^{\top}_{\text{on}}\textbf{C}^{-1}\textbf{A}_{\text{on}}.
\end{equation*}
The key matrix $\textbf{A}_{\text{TDC}}$ of off-policy TDC is
\begin{equation*}
 \textbf{A}_{\text{TDC}} = \textbf{A}^{\top}_{\text{off}}\textbf{C}^{-1}\textbf{A}_{\text{off}}.
\end{equation*}
$\textbf{A}_{\text{TDC}}=0.016$ in the off-policy 2-state and $\textbf{A}_{\text{TDC}}=0.09025$ 
in the on-policy 2-state, which means that TDC can stably converge in two settings.

To address the issue of the key matrix $\textbf{A}_{\text{off}}$
in off-policy TD(0) being non-positive definite,
 a scalar variable, $F_t$,  
is introduced to obtain the off-policy TD(0) algorithm, 
which ensures convergence under off-policy 
conditions.

The key matrix $\textbf{A}_{\text{ETD}}$ is
\begin{equation*}
 \textbf{A}_{\text{ETD}} = \bm{\Phi}^{\top}\textbf{F}(\textbf{I}-\gamma \textbf{P}_{\pi})\bm{\Phi},
\end{equation*}
where 
$\textbf{F}$ is a diagonal matrix with diagonal elements
$f(s)\dot{=}d_{\mu}(s)\lim_{t\rightarrow \infty}\mathbb{E}_{\mu}[F_t|S_t=s]$,
which we assume exists. 
The vector $\textbf{f}\in \mathbb{R}^N$ with components 
$[\textbf{f}]_s\dot{=}f(s)$ can be written as
\begin{equation*}
\begin{split}
\textbf{f}&=\textbf{d}_{\mu}+\gamma \textbf{P}_{\pi}^{\top}\textbf{d}_{\mu}+(\gamma \textbf{P}_{\pi}^{\top})^2\textbf{d}_{\mu}+\ldots\\
&=(\textbf{I}-\gamma\textbf{P}_{\pi}^{\top})^{-1}\textbf{d}_{\mu}.
\end{split}
\end{equation*}
In the off-policy 2-state, 
$\textbf{A}_{\text{ETD}}=3.4$ and $\textbf{A}_{\text{ETD}}=4.75$ for on-policy, which means that ETD can stably converge.

Table \ref{tab:min_eigenvalues} shows Minimum eigenvalues 
of various algorithms in the 2-state counterexample.

In both the on-policy 2-state environment 
and the off-policy 2-state environment, 
the minimum eigenvalue of the key matrix for 
ETD is larger than that of TD(0) and TDC, indicating 
that ETD has the fastest convergence rate.

Minimum eigenvalue larger, the algorithm's convergence faster.
To derive an algorithm with a larger minimum eigenvalue for the matrix 
$\textbf{A}$, it is necessary to propose new objective functions. 
The mentioned objective functions in the Introduction 
are all forms of error. Is minimizing error the only option 
for value-based reinforcement learning?
Based on this observation, 
we propose alternative objective functions instead of minimizing errors.

\section{Variance Minimization Algorithms}
This section will introduce two new objective functions and 
three new algorithms, including one on-policy algorithm and two off-policy algorithms, and calculate the minimum eigenvalue 
of $\textbf{A}$ for each of the three algorithms under on-policy and 
off-policy in a 2-state environment, thereby comparing the 
convergence speed of the three algorithms.

\subsection{Variance Minimization TD Learning: VMTD}
For on-policy learning,
a novel objective function, Variance of Bellman Error (VBE), is proposed as follows:

\begin{align}
 \arg \min_{\theta}\text{VBE}(\theta) &= \arg \min_{\theta}\mathbb{E}[(\mathbb{E}[\delta_t|S_t]-\mathbb{E}[\mathbb{E}[\delta_t|S_t]])^2] \\ 
 &= \arg \min_{\theta,\omega} \mathbb{E}[(\mathbb{E}[\delta_t|S_t]-\omega)^2]\notag
\end{align}
where $\delta_t$ is the TD error as follows:
\begin{equation}
\delta_t = r_{t+1}+\gamma
\theta_t^{\top}\phi_{t+1}-\theta_t^{\top}\phi_t.
\label{delta}
\end{equation}
Clearly, it is no longer to minimize Bellman errors. 

First, the parameter  $\omega$ is derived directly based on
stochastic gradient descent:
\begin{equation}
\omega_{t+1}\leftarrow \omega_{t}+\beta_t(\delta_t-\omega_t),
\label{omega}
\end{equation}

Then, based on stochastic semi-gradient descent, the update of 
the parameter $\theta$ is as follows:
\begin{equation}
\theta_{t+1}\leftarrow
\theta_{t}+\alpha_t(\delta_t-\omega_t)\phi_t.
\label{theta}
\end{equation}

The semi-gradient of the (2) with respect to $\theta$ is
\begin{equation*}
 \begin{array}{ccl}
 &&-\frac{1}{2}\nabla \text{VBE}({\theta}) \\
 &=& \mathbb{E}[(\mathbb{E}[\delta_t|S_t]-\mathbb{E}[\mathbb{E}[\delta_t|S_t]])(\phi_t -\mathbb{E}[\phi_t])] \\
 &=& \mathbb{E}[\delta_t \phi_t] -\mathbb{E}[\delta_t] \mathbb{E}[\phi_t] , 
 % &=&-\mathbb{E}\Big[\Big( (\phi_t - \gamma\phi_t')- \mathbb{E}[ (\phi_t- \gamma {\phi_t}')]\Big)\phi_t^{\top} \Big]\theta + \mathbb{E}( r_{t+1}- \mathbb{E}[r_{t+1}])\bm{\phi_t}
\end{array}
\end{equation*}
The key matrix $\textbf{A}_{\text{VMTD}}$ and $b_{\text{VMTD}}$ of on-policy VMTD is
\begin{equation*}
 \begin{array}{ccl}
 &&\textbf{A}_{\text{VMTD}} \\
 &=& \mathbb{E}[(\phi - \gamma \phi')\phi^{\top}] - \mathbb{E}[\phi - \gamma \phi']\mathbb{E}[\phi^{\top}]\\
 &=&\sum_{s}d_{\pi}(s)\phi(s)\Big(\phi(s) -\gamma  \sum_{s'}[\textbf{P}_{\pi}]_{ss'}\phi(s') \Big)^{\top} \\
 && -\sum_{s}d_{\pi}(s)\phi(s) \cdot \sum_{s}d_{\pi}(s)\Big(\phi(s) -\gamma \sum_{s'}[\textbf{P}_{\pi}]_{ss'}\phi(s') \Big)^{\top}\\
 &=& \bm{\Phi}^{\top}\textbf{D}_{\mu}(\textbf{I}-\gamma \textbf{P}_{\pi})\bm{\Phi} -\bm{\Phi}^{\top}d_{\pi}d_{\pi}^{\top}(\textbf{I}-\gamma \textbf{P}_{\pi})\bm{\Phi}\\
 &=& \bm{\Phi}^{\top}(\textbf{D}_{\pi}-d_{\pi}d_{\pi}^{\top})(\textbf{I}-\gamma \textbf{P}_{\pi})\bm{\Phi},
\end{array}
\end{equation*}
\begin{equation*}
 \begin{array}{ccl}
  &&b_{\text{VMTD}}\\ 
  &=& \mathbb{E}( r- \mathbb{E}[r])\phi  \\
  &=& \mathbb{E}[r\phi] - \mathbb{E}[r]\mathbb{E}[\phi]\\
   &=& \bm{\Phi}^{\top}(\textbf{D}_{\pi}-d_{\pi}d_{\pi}^{\top})r_\pi.
 \end{array}
\end{equation*}
It can be easily obtained that The key matrix $\textbf{A}_{\text{VMTD}}$ and $b_{\text{VMTD}}$ of off-policy VMTD are, respectively, 
\begin{equation*}
 \textbf{A}_{\text{VMTD}} = \bm{\Phi}^{\top}(\textbf{D}_{\mu}-d_{\mu}d_{\mu}^{\top})(\textbf{I}-\gamma \textbf{P}_{\pi})\bm{\Phi},
 \end{equation*}
 \begin{equation*}
 b_{\text{VMTD}}=\bm{\Phi}^{\top}(\textbf{D}_{\mu}-d_{\mu}d_{\mu}^{\top})r_\pi,
 \end{equation*}
 In the on-policy 2-state environment, the minimum eigenvalue 
of the key matrix for VMTD is greater than that of TDC and smaller than that of TD(0) and ETD, 
indicating that VMTD converges faster than TDC and slower than TD(0) and ETD in this 
environment. In the off-policy 2-state environment, 
the minimum eigenvalue of the key matrix for VMTD is 
greater than 0, suggesting that VMTD can converge stably, while TD(0) diverges.

\subsection{Variance Minimization TDC Learning: VMTDC}
For off-policy learning, we propose a new objective function, 
called Variance of Projected Bellman error (VPBE), 
and the corresponding algorithm is called VMTDC.
\begin{align}
 &\text{VPBE}(\bm{\theta}) \notag\\
 &= \mathbb{E}[(\delta-\mathbb{E}[\delta]) \bm{\phi}]^{\top} \mathbb{E}[\bm{\phi} \bm{\phi}^{\top}]^{-1}\mathbb{E}[(\delta -\mathbb{E}[\delta ])\bm{\phi}] \\ 
 &= \mathbb{E}[(\delta-\omega) \bm{\phi}]^{\top} \mathbb{E}[\bm{\phi} \bm{\phi}^{\top}]^{-1}\mathbb{E}[(\delta -\omega)\bm{\phi}] ,
\end{align}
where 
% $\textbf{W}_{\bm{\theta}}$ is viewed as vectors with every element being equal to 
% $||\textbf{V}_{\bm{\theta}} - \textbf{T}\textbf{V}_{\bm{\theta}}||^{2}_{\mu}$ and 
$\omega$ is used to approximate $\mathbb{E}[\delta]$, i.e., $\omega \doteq\mathbb{E}[\delta] $.

The gradient of the (6) with respect to $\theta$ is
\begin{equation*}
 \begin{array}{ccl}
 -\frac{1}{2}\nabla \text{VPBE}({\theta}) &=& -\mathbb{E}\Big[\Big( (\gamma {\phi}' - {\phi}) - \mathbb{E}[ (\gamma {\phi}' - {\phi})]\Big){\phi}^{\top} \Big] \\
 & & \mathbb{E}[{\phi} {\phi}^{\top}]^{-1} \mathbb{E}[( \delta -\mathbb{E}[ \delta]){\phi}]\\
 &=& \mathbb{E}\Big[\Big( ({\phi} - \gamma {\phi}')- \mathbb{E}[ ({\phi} - \gamma {\phi}')]\Big){\phi}^{\top} \Big]  \\
 & & \mathbb{E}[{\phi} {\phi}^{\top}]^{-1}\\
 & & \mathbb{E}\Big[\Big( r + \gamma {{\phi}'}^{\top} {\theta} -{\phi}^{\top} {\theta}\\
 & & \hspace{2em} -\mathbb{E}[ r + \gamma {{\phi}'}^{\top} {\theta} -{\phi}^{\top} {\theta}]\Big){\phi} \Big].
 \end{array}
\end{equation*}
It can be easily obtained that The key matrix $\textbf{A}_{\text{VMTDC}}$ and $b_{\text{VMTDC}}$ of VMTDC are, respectively, 
\begin{equation*}
 \textbf{A}_{\text{VMTDC}} = \textbf{A}_{\text{VMTD}}^{\top} \textbf{C}^{-1}\textbf{A}_{\text{VMTD}},
 \end{equation*}
 \begin{equation*}
 b_{\text{VMTDC}}=\textbf{A}_{\text{VMTD}}^{\top} \textbf{C}^{-1}b_{\text{VMTD}},
 \end{equation*}
where, for on-policy, $\textbf{A}_{\text{VMTD}}=\bm{\Phi}^{\top}(\textbf{D}_{\pi}-d_{\pi}d_{\pi}^{\top})(\textbf{I}-\gamma \textbf{P}_{\pi})\bm{\Phi}$ 
and $b_{\text{VMTD}}=\bm{\Phi}^{\top}(\textbf{D}_{\pi}-d_{\pi}d_{\pi}^{\top})r_\pi$ and, for off-policy, 
$\textbf{A}_{\text{VMTD}}=\bm{\Phi}^{\top}(\textbf{D}_{\mu}-d_{\mu}d_{\mu}^{\top})(\textbf{I}-\gamma \textbf{P}_{\pi})\bm{\Phi}$ 
and $b_{\text{VMTD}}=\bm{\Phi}^{\top}(\textbf{D}_{\mu}-d_{\mu}d_{\mu}^{\top})r_\pi$.

In the process of computing the gradient of the (7) with respect to $\theta$, 
$\omega$ is treated as a constant.
So, the derivation process of the VMTDC algorithm is the same 
as for the TDC algorithm, the only difference is that the original $\delta$ is replaced by $\delta-\omega$.
Therefore, we can easily get the updated formula of VMTDC, as follows:
\begin{equation}
 \bm{\theta}_{k+1}\leftarrow\bm{\theta}_{k}+\alpha_{k}[(\delta_{k}- \omega_k) \bm{\phi}_k\\
 - \gamma\bm{\phi}_{k+1}(\bm{\phi}^{\top}_k \bm{u}_{k})],
\label{thetavmtdc}
\end{equation}
\begin{equation}
 \bm{u}_{k+1}\leftarrow \bm{u}_{k}+\zeta_{k}[\delta_{k}-\omega_k - \bm{\phi}^{\top}_k \bm{u}_{k}]\bm{\phi}_k,
\label{uvmtdc}
\end{equation}
and
\begin{equation}
 \omega_{k+1}\leftarrow \omega_{k}+\beta_k (\delta_k- \omega_k).
 \label{omegavmtdc}
\end{equation}
The VMTDC algorithm (\ref{thetavmtdc}) is derived to work 
with a given set of sub-samples—in the form of 
triples $(S_k, R_k, S'_k)$ that match transitions 
from both the behavior and target policies.

In the on-policy 2-state environment, the minimum eigenvalue 
of the key matrix for VMTDC is smaller than that of TD(0), TDC, ETD, and VMTD 
indicating that VMTDC converges slower than them in this 
on-policy. In the off-policy 2-state environment, the 
the minimum eigenvalue of the key matrix for VMTDC is greater than TDC, 
suggesting that VMTDC converges faster than TDC in off-policy 
environment.

\subsection{Variance Minimization ETD Learning: VMETD}
Based on the off-policy TD algorithm, a scalar, $F$,  
is introduced to obtain the ETD algorithm, 
which ensures convergence under off-policy 
conditions. This paper further introduces a scalar, 
$\omega$, based on the ETD algorithm to obtain VMETD.
VMETD by the following update:
\begin{equation}
\label{fvmetd}
 F_t \leftarrow \gamma \rho_{t-1}F_{t-1}+1,
\end{equation}
\begin{equation}
 \label{thetavmetd}
 {\theta}_{t+1}\leftarrow {\theta}_t+\alpha_t (F_t \rho_t\delta_t - \omega_{t}){\phi}_t,
\end{equation}
\begin{equation}
 \label{omegavmetd}
 \omega_{t+1} \leftarrow \omega_t+\beta_t(F_t  \rho_t \delta_t - \omega_t),
\end{equation}
where $\rho_t =\frac{\pi(A_t | S_t)}{\mu(A_t | S_t)}$ and $\omega$ is used to estimate $\mathbb{E}[F \rho\delta]$, i.e., $\omega \doteq \mathbb{E}[F \rho\delta]$.

(\ref{thetavmetd}) can be rewritten as
\begin{equation*}
 \begin{array}{ccl}
 {\theta}_{t+1}&\leftarrow& {\theta}_t+\alpha_t (F_t \rho_t\delta_t - \omega_t){\phi}_t -\alpha_t \omega_{t+1}{\phi}_t\\
  &=&{\theta}_{t}+\alpha_t(F_t\rho_t\delta_t-\mathbb{E}_{\mu}[F_t\rho_t\delta_t|{\theta}_t]){\phi}_t\\
 &=&{\theta}_t+\alpha_t F_t \rho_t (r_{t+1}+\gamma {\theta}_t^{\top}{\phi}_{t+1}-{\theta}_t^{\top}{\phi}_t){\phi}_t\\
 & & \hspace{2em} -\alpha_t \mathbb{E}_{\mu}[F_t \rho_t \delta_t]{\phi}_t\\
 &=& {\theta}_t+\alpha_t \{\underbrace{(F_t\rho_tr_{t+1}-\mathbb{E}_{\mu}[F_t\rho_t r_{t+1}]){\phi}_t}_{{b}_{\text{VMETD},t}}\\
 &&\hspace{-7em}- \underbrace{(F_t\rho_t{\phi}_t({\phi}_t-\gamma{\phi}_{t+1})^{\top}-{\phi}_t\mathbb{E}_{\mu}[F_t\rho_t ({\phi}_t-\gamma{\phi}_{t+1})]^{\top})}_{\textbf{A}_{\text{VMETD},t}}{\theta}_t\}.
 \end{array}
\end{equation*}
Therefore, 
\begin{equation*}
 \begin{array}{ccl}
  &&\textbf{A}_{\text{VMETD}}\\
  &=&\lim_{t \rightarrow \infty} \mathbb{E}[\textbf{A}_{\text{VMETD},t}]\\
  &=& \lim_{t \rightarrow \infty} \mathbb{E}_{\mu}[F_t \rho_t {\phi}_t ({\phi}_t - \gamma {\phi}_{t+1})^{\top}]\\  
  &&\hspace{1em}- \lim_{t\rightarrow \infty} \mathbb{E}_{\mu}[  {\phi}_t]\mathbb{E}_{\mu}[F_t \rho_t ({\phi}_t - \gamma {\phi}_{t+1})]^{\top}\\
  &=& \lim_{t \rightarrow \infty} \mathbb{E}_{\mu}[{\phi}_tF_t \rho_t ({\phi}_t - \gamma {\phi}_{t+1})^{\top}]\\   
  &&\hspace{1em}-\lim_{t \rightarrow \infty} \mathbb{E}_{\mu}[ {\phi}_t]\lim_{t \rightarrow \infty}\mathbb{E}_{\mu}[F_t \rho_t ({\phi}_t - \gamma {\phi}_{t+1})]^{\top}\\
  && \hspace{-2em}=\sum_{s} d_{\mu}(s)\lim_{t \rightarrow \infty}\mathbb{E}_{\mu}[F_t|S_t = s]\mathbb{E}_{\mu}[\rho_t\phi_t(\phi_t - \gamma \phi_{t+1})^{\top}|S_t= s]\\   
  &&\hspace{1em}-\sum_{s} d_{\mu}(s)\phi(s)\sum_{s} d_{\mu}(s)\lim_{t \rightarrow \infty}\mathbb{E}_{\mu}[F_t|S_t = s]\\
  &&\hspace{7em}\mathbb{E}_{\mu}[\rho_t(\phi_t - \gamma \phi_{t+1})^{\top}|S_t = s]\\
  &=& \sum_{s} f(s)\mathbb{E}_{\pi}[\phi_t(\phi_t- \gamma \phi_{t+1})^{\top}|S_t = s]\\   
  &&\hspace{1em}-\sum_{s} d_{\mu}(s)\phi(s)\sum_{s} f(s)\mathbb{E}_{\pi}[(\phi_t- \gamma \phi_{t+1})^{\top}|S_t = s]\\
  &=&\sum_{s} f(s) \bm{\phi}(s)(\bm{\phi}(s) - \gamma \sum_{s'}[\textbf{P}_{\pi}]_{ss'}\bm{\phi}(s'))^{\top}  \\
  &&\hspace{1em} -\sum_{s} d_{\mu}(s) {\phi}(s) * \sum_{s} f(s)({\phi}(s) - \gamma \sum_{s'}[\textbf{P}_{\pi}]_{ss'}{\phi}(s'))^{\top}\\
  &=&{\bm{\Phi}}^{\top} \textbf{F} (\textbf{I} - \gamma \textbf{P}_{\pi}) \bm{\Phi} - {\bm{\Phi}}^{\top} {d}_{\mu} {f}^{\top} (\textbf{I} - \gamma \textbf{P}_{\pi}) \bm{\Phi}  \\
  &=&{\bm{\Phi}}^{\top} (\textbf{F} - {d}_{\mu} {f}^{\top}) (\textbf{I} - \gamma \textbf{P}_{\pi}){\bm{\Phi}} \\
  &=&{\bm{\Phi}}^{\top} (\textbf{F} (\textbf{I} - \gamma \textbf{P}_{\pi})-{d}_{\mu} {f}^{\top} (\textbf{I} - \gamma \textbf{P}_{\pi})){\bm{\Phi}} \\
  &=&{\bm{\Phi}}^{\top} (\textbf{F} (\textbf{I} - \gamma \textbf{P}_{\pi})-{d}_{\mu} {d}_{\mu}^{\top} ){\bm{\Phi}},
 \end{array}
\end{equation*}
\begin{equation*}
 \begin{array}{ccl}
  &&{b}_{\text{VMETD}}\\
  &=&\lim_{t \rightarrow \infty} \mathbb{E}[{b}_{\text{VMETD},t}]\\
  &=& \lim_{t \rightarrow \infty} \mathbb{E}_{\mu}[F_t\rho_tR_{t+1}{\phi}_t]\\
  &&\hspace{2em} - \lim_{t\rightarrow \infty} \mathbb{E}_{\mu}[{\phi}_t]\mathbb{E}_{\mu}[F_t\rho_kR_{k+1}]\\  
  &=& \lim_{t \rightarrow \infty} \mathbb{E}_{\mu}[{\phi}_tF_t\rho_tr_{t+1}]\\
  &&\hspace{2em} - \lim_{t\rightarrow \infty} \mathbb{E}_{\mu}[  {\phi}_t]\mathbb{E}_{\mu}[{\phi}_t]\mathbb{E}_{\mu}[F_t\rho_tr_{t+1}]\\ 
  &=& \lim_{t \rightarrow \infty} \mathbb{E}_{\mu}[{\phi}_tF_t\rho_tr_{t+1}]\\
  &&\hspace{2em} - \lim_{t \rightarrow \infty} \mathbb{E}_{\mu}[ {\phi}_t]\lim_{t \rightarrow \infty}\mathbb{E}_{\mu}[F_t\rho_tr_{t+1}]\\  
  &=&\sum_{s} f(s) {\phi}(s)r_{\pi} - \sum_{s} d_{\mu}(s) {\phi}(s) * \sum_{s} f(s)r_{\pi}  \\
  &=&\bm{\bm{\Phi}}^{\top}(\textbf{F}-{d}_{\mu} {f}^{\top}){r}_{\pi}.
 \end{array}
\end{equation*}
In both the off-policy 2-state environment and the on-policy 2-state environment, the minimum eigenvalue 
of the key matrix for VMETD is greater than that of TD(0), TDC, VMTD, and VMTDC and smaller than that of ETD,
indicating that VMTDC converges faster than TD(0), TDC, VMTD, and VMTDC and slower than ETD. 
However, subsequent experiments showed that the VMETD algorithm converges more smoothly and performs best in controlled experiments.

\section{Theoretical Analysis}
This section primarily focuses on proving the convergence of VMTD, VMTDC, and VMETD.
\begin{theorem}
 \label{theorem1}(Convergence of VMTD).
 In the case of on-policy learning, consider the iterations (\ref{omega}) and (\ref{theta}) with (\ref{delta})  of VMTD.
 Let the step-size sequences $\alpha_k$ and $\beta_k$, $k\geq 0$ satisfy in this case $\alpha_k,\beta_k>0$, for all $k$,
 $
 \sum_{k=0}^{\infty}\alpha_k=\sum_{k=0}^{\infty}\beta_k=\infty,
 $
 $
 \sum_{k=0}^{\infty}\alpha_k^2<\infty,
 $
 $
 \sum_{k=0}^{\infty}\beta_k^2<\infty,
 $
 and  
 $
 \alpha_k = o(\beta_k).
 $
 Assume that $(\phi_k,r_k,\phi_k')$ is an i.i.d. sequence with
 uniformly bounded second moments, where $\phi_k$ and $\phi'_{k}$ are sampled from the same Markov chain.
 Let $\textbf{A} = \mathrm{Cov}(\phi,\phi-\gamma\phi')$,
 $b=\mathrm{Cov}(r,\phi)$.
 Assume that matrix $\textbf{A}$ is non-singular. 
 Then the parameter vector $\theta_k$ converges with probability one 
 to $\textbf{A}^{-1}b$.
\end{theorem}

\begin{proof}
\label{th1proof}   
 The proof is  based on Borkar's Theorem   for
 general stochastic approximation recursions with two time scales
 \cite{borkar1997stochastic}. 

 A sketch proof is given as follows.
 In the fast time scale, the parameter $w$ converges to
  $\mathbb{E}[\delta|\theta_k]$.
 In the slow time scale,
 the associated ODE is
  \begin{equation*}
 \vec{h}(\theta(t))=-\textbf{A}\theta(t)+b.
 \end{equation*}
 \begin{equation}
  \begin{array}{ccl}
 A &=& \mathrm{Cov}(\phi,\phi-\gamma\phi')\\
    &=&\frac{\mathrm{Cov}(\phi,\phi)+\mathrm{Cov}(\phi-\gamma\phi',\phi-\gamma\phi')-\mathrm{Cov}(\gamma\phi',\gamma\phi')}{2}\\
    &=&\frac{\mathrm{Cov}(\phi,\phi)+\mathrm{Cov}(\phi-\gamma\phi',\phi-\gamma\phi')-\gamma^2\mathrm{Cov}(\phi',\phi')}{2}\\
    &=&\frac{(1-\gamma^2)\mathrm{Cov}(\phi,\phi)+\mathrm{Cov}(\phi-\gamma\phi',\phi-\gamma\phi')}{2},\\
  \end{array}
  \label{covariance}
  \end{equation}
 where we eventually used $\mathrm{Cov}(\phi',\phi')=\mathrm{Cov}(\phi,\phi)$
  \footnote{The covariance matrix $\mathrm{Cov}(\phi',\phi')$ is equal to
 the covariance matrix $\mathrm{Cov}(\phi,\phi)$ if the initial state is re-reachable or
 initialized randomly in a Markov chain for on-policy update.}.
 Note that the covariance matrix $\mathrm{Cov}(\phi,\phi)$ and
  $\mathrm{Cov}(\phi-\gamma\phi',\phi-\gamma\phi')$ are semi-positive
 definite. Then, the matrix $\textbf{A}$ is semi-positive definite because  $\textbf{A}$ is
 linearly combined  by  two positive-weighted semi-positive definite matrice
 (\ref{covariance}).
 Furthermore, $\textbf{A}$ is nonsingular due to the assumption.
 Hence, the matrix $\textbf{A}$ is positive definite. And,
 the parameter $\theta$ converges to $\textbf{A}^{-1}b$.
\end{proof}
Please refer to the appendix for VMTD's detailed proof process.

\begin{theorem}
 \label{theorem2}(Convergence of VMTDC).
 In the case of off-policy learning, consider the iterations (\ref{omegavmtdc}), (\ref{uvmtdc}) and (\ref{thetavmtdc})   of VMTDC.
 Let the step-size sequences $\alpha_k$, $\zeta_k$ and $\beta_k$, $k\geq 0$ satisfy in this case $\alpha_k,\zeta_k,\beta_k>0$, for all $k$,
 $
 \sum_{k=0}^{\infty}\alpha_k=\sum_{k=0}^{\infty}\beta_k=\sum_{k=0}^{\infty}\zeta_k=\infty,
 $
 $
 \sum_{k=0}^{\infty}\alpha_k^2<\infty,
 $
 $
 \sum_{k=0}^{\infty}\zeta_k^2<\infty,
 $
 $
 \sum_{k=0}^{\infty}\beta_k^2<\infty,
 $
 and  
 $
 \alpha_k = o(\zeta_k),
 $
 $
 \zeta_k = o(\beta_k).
 $
 Assume that $(\phi_k,r_k,\phi_k')$ is an i.i.d. sequence with
 uniformly bounded second moments.
 Let $\textbf{A} = \mathrm{Cov}(\phi,\phi-\gamma\phi')$,
 $b=\mathrm{Cov}(r,\phi)$, and $\textbf{C}=\mathbb{E}[\phi\phi^{\top}]$.
 Assume that  $\textbf{A}$ and $\textbf{C}$ are non-singular matrices. 
 Then the parameter vector $\theta_k$ converges with probability one 
 to $\textbf{A}^{-1}b$.
\end{theorem}
\begin{proof}
 The proof is similar to that given by \cite{sutton2009fast} for TDC,
 but it is based on multi-time-scale stochastic approximation.
 
 A sketch proof is given as follows.
 In the fastest time scale, the parameter $w$ converges to
  $\mathbb{E}[\delta|u_k,\theta_k]$.
 In the second fast time scale,
 the parameter $u$ converges to $\textbf{C}^{-1}\mathbb{E}[(\delta-\mathbb{E}[\delta|\theta_k])\phi|\theta_k]$.
 In the slower time scale,
 the associated ODE is
  \begin{equation*}
 \vec{h}(\theta(t))=\textbf{A}^{\top}\textbf{C}^{-1}(-\textbf{A}\theta(t)+b).
 \end{equation*}
 The matrix $\textbf{A}^{\top}\textbf{C}^{-1}\textbf{A}$ is positive definite. Thus,
 the parameter $\theta$ converges to $\textbf{A}^{-1}b$.
\end{proof}
Please refer to the appendix for VMTDC's detailed proof process.

\begin{theorem}
 \label{theorem3}(Convergence of VMETD).
 In the case of off-policy learning, consider the iterations (\ref{fvmetd}), (\ref{omegavmetd}), and (\ref{thetavmetd}) of VMETD.
 Let the step-size sequences $\alpha_k$ and $\beta_k$, $k\geq 0$ satisfy in this case $\alpha_k,\beta_k>0$, for all $k$,
 $
 \sum_{k=0}^{\infty}\alpha_k=\sum_{k=0}^{\infty}\beta_k=\infty,
 $
 $
 \sum_{k=0}^{\infty}\alpha_k^2<\infty,
 $
 $
 \sum_{k=0}^{\infty}\beta_k^2<\infty,
 $
 and  
 $
 \alpha_k = o(\beta_k).
 $
 Assume that $(\bm{\bm{\phi}}_k,r_k,\bm{\bm{\phi}}_k')$ is an i.i.d. sequence with
 uniformly bounded second moments, where $\bm{\bm{\phi}}_k$ and $\bm{\bm{\phi}}'_{k}$ are sampled from the same Markov chain.
 Let $\textbf{A}_{\textbf{VMETD}} ={\bm{\Phi}}^{\top} (\textbf{F} (\textbf{I} - \gamma \textbf{P}_{\pi})-d_{\mu} d_{\mu}^{\top} ){\bm{\Phi}}$,
 $b_{\text{VMETD}}=\bm{\Phi}^{\top}(\textbf{F}-d_{\mu} f^{\top})r_{\pi}$.
 Assume that matrix $\textbf{A}$ is non-singular. 
 Then the parameter vector $\theta_k$ converges with probability one 
 to $\textbf{A}_{\textbf{VMETD}}^{-1}b_{\textbf{VMETD}}$.
\end{theorem}
\begin{proof}
 The proof of VMETD's convergence is  based on Borkar's Theorem   for
 general stochastic approximation recursions with two time scales
 \cite{borkar1997stochastic}. 

 A sketch proof is given as follows.
 In the fast time scale, the parameter $\omega$ converges to
  $\mathbb{E}_{\mu}[F\rho\delta|\theta_k]$.
Recursion (\ref{thetavmetd}) is considered the slower timescale. 
If the key matrix 
$\textbf{A}_{\text{VMETD}}$ is positive definite, then 
$\theta$ converges.

An ${\Phi}^{\top}{\text{X}}{\Phi}$ matrix of this
 form will be positive definite whenever the matrix ${\text{X}}$ is positive definite.
 Any matrix ${\text{X}}$ is positive definite if and only if
 the symmetric matrix ${\text{S}}={\text{X}}+{\text{X}}^{\top}$ is positive definite. 
 Any symmetric real matrix ${\text{S}}$ is positive definite if the absolute values of
 its diagonal entries are greater than the sum of the absolute values of the corresponding
 off-diagonal entries\cite{sutton2016emphatic}. 
\begin{equation}
 \label{rowsum}
 \begin{split}
 &(\textbf{F} (\textbf{I} - \gamma \textbf{P}_{\pi})-{d}_{\mu} {d}_{\mu}^{\top} )e\\
 &=\textbf{F} (\textbf{I} - \gamma \textbf{P}_{\pi})e-{d}_{\mu} {d}_{\mu}^{\top} e\\
 % &=\textbf{F}(\textbfe-\gamma \textbf{P}_{\pi} \textbfe)-\textbf{d}_{\mu} \textbf{d}_{\mu}^{\top} \textbfe\\
 &=(1-\gamma)\textbf{F}e-{d}_{\mu} {d}_{\mu}^{\top} e\\
 % &=(1-\gamma)\textbf{f}-\textbf{d}_{\mu} \textbf{d}_{\mu}^{\top} \textbfe\\
 &=(1-\gamma){f}-{d}_{\mu} \\
 &=(1-\gamma)(\textbf{I}-\gamma\textbf{P}_{\pi}^{\top})^{-1}{d}_{\mu}-{d}_{\mu} \\
 &=(1-\gamma)[(\textbf{I}-\gamma\textbf{P}_{\pi}^{\top})^{-1}-\textbf{I}]{d}_{\mu} \\
 &=(1-\gamma)[\sum_{t=0}^{\infty}(\gamma\textbf{P}_{\pi}^{\top})^{t}-\textbf{I}]{d}_{\mu} \\
 &=(1-\gamma)[\sum_{t=1}^{\infty}(\gamma\textbf{P}_{\pi}^{\top})^{t}]{d}_{\mu} > 0, \\
 \end{split}
 \end{equation}
\begin{equation}
 \label{columnsum}
 \begin{split}
 &e^{\top}(\textbf{F} (\textbf{I} - \gamma \textbf{P}_{\pi})-\textbf{d}_{\mu} {d}_{\mu}^{\top} )\\
 &=e^{\top}\textbf{F} (\textbf{I} - \gamma \textbf{P}_{\pi})-e^{\top}{d}_{\mu} {d}_{\mu}^{\top} \\
 &={d}_{\mu}^{\top}-e^{\top}{d}_{\mu} {d}_{\mu}^{\top} \\
 &={d}_{\mu}^{\top}- {d}_{\mu}^{\top} \\
 &=0,
 \end{split}
\end{equation}
where $e$ is the all-ones vector.
(\ref{rowsum}) and (\ref{columnsum}) show that the matrix $\textbf{F} (\textbf{I} - \gamma \textbf{P}_{\pi})-d_{\mu} d_{\mu}^{\top}$ of
 diagonal entries are positive and its off-diagonal entries are negative. So each row sum plus the corresponding column sum is positive. 
So $\textbf{A}_{\text{VMETD}}$ is positive definite.
\end{proof}

\subsection{On the Fixed-point Solutions}
The fixed-point solutions of 
VMTD, VMTDC, and VMETD are $\textbf{A}^{-1}b$, $\textbf{A}^{-1}b$, 
and $\textbf{A}_{\text{VMETD}}^{-1}b_{\text{VMETD}}$, 
respectively,
 which differ from the traditional TD fixed-point solutions. 
 Therefore, this paper is concerned with the impact of these 
 three VM algorithms' solutions on the policy, specifically the policy 
 invariance of the VM algorithms. Before proving the policy invariance 
 of the algorithms, we will first discuss reward shaping.

Reward shaping can significantly speed up learning  by adding a shaping
reward $F(s,s')$ to the original reward  $r$, 
where $F(s,s')$ is the general form of any state-based shaping reward.
Static potential-based reward shaping (Static PBRS) maintains the policy invariance if the
shaping reward follows from $F(s,s')=\gamma
f(s')-f(s)$ \cite{ng1999policy}.

This means that we can make changes to the TD error $\delta = r+\gamma {\theta}^{\top}{\phi}'-{\theta}^{\top}{\phi} $ while still ensuring the invariance of the optimal policy,
\begin{equation*}
 \delta - \omega= r+\gamma {\theta}^{\top}{\phi}'-{\theta}^{\top}{\phi} - \omega,
\end{equation*}
where $\omega$ is a constant, acting as a static PBRS. 
However, selecting an appropriate 
$\omega$ requires expert knowledge. This forces us to learn 
$\omega$ dynamically, i.e., $\omega=\omega_t $ and dynamic PBRS can maintain the policy 
invariance if the shaping reward is $F(s,t,s',t')=\gamma f(s',t')-f(s,t)$,
where $t$ is the time-step the agent reaches in  state $s$
\cite{devlin2012dynamic}.
However, this result requires the convergence guarantee of the dynamic potential
function $f(s,t)$. If $f(s,t)$ does not converge as the time-step
$t\rightarrow\infty$, the Q-values of dynamic PBRS are not 
guaranteed to converge.

\begin{theorem} (Optimal policy invariance of VMTD).
 Consider the iterations (\ref{omega}) and (\ref{theta}) with (\ref{delta})  of VMTD. 
 The VMTD algorithm maintains
 the optimal policy invariance.
 \end{theorem}
 \begin{proof}
 Consider the parameter variable
 \begin{equation*}
 \omega_t  
 =\gamma (\frac{\omega_t}{\gamma-1})-(\frac{\omega_t}{\gamma-1}).
 \end{equation*}
 Let $\eta (s)\equiv \frac{1}{\gamma-1}$,
 consequently, 
 \begin{equation*}
 f_{\omega_t}(s, t)=\omega_t^{\top}\zeta(s)=\frac{\omega_t}{\gamma-1}
 \end{equation*}
 
 Thus, $F_{\omega_t}(s,s')=\gamma f_{\omega_t}(s')-f_{\omega_t}(s)= \omega_t$
 is a dynamic potential-based function. In VMTD, $\omega_t$ has been proven to be convergent.
\end{proof}

Similarly, both VMTDC and VMETD can ensure optimal policy invariance.
\begin{figure*}[tb]
     \vskip 0.2in
     \begin{center}
     \subfigure[on-policy 2-state]{
      \includegraphics[width=0.65\columnwidth, height=0.58\columnwidth]{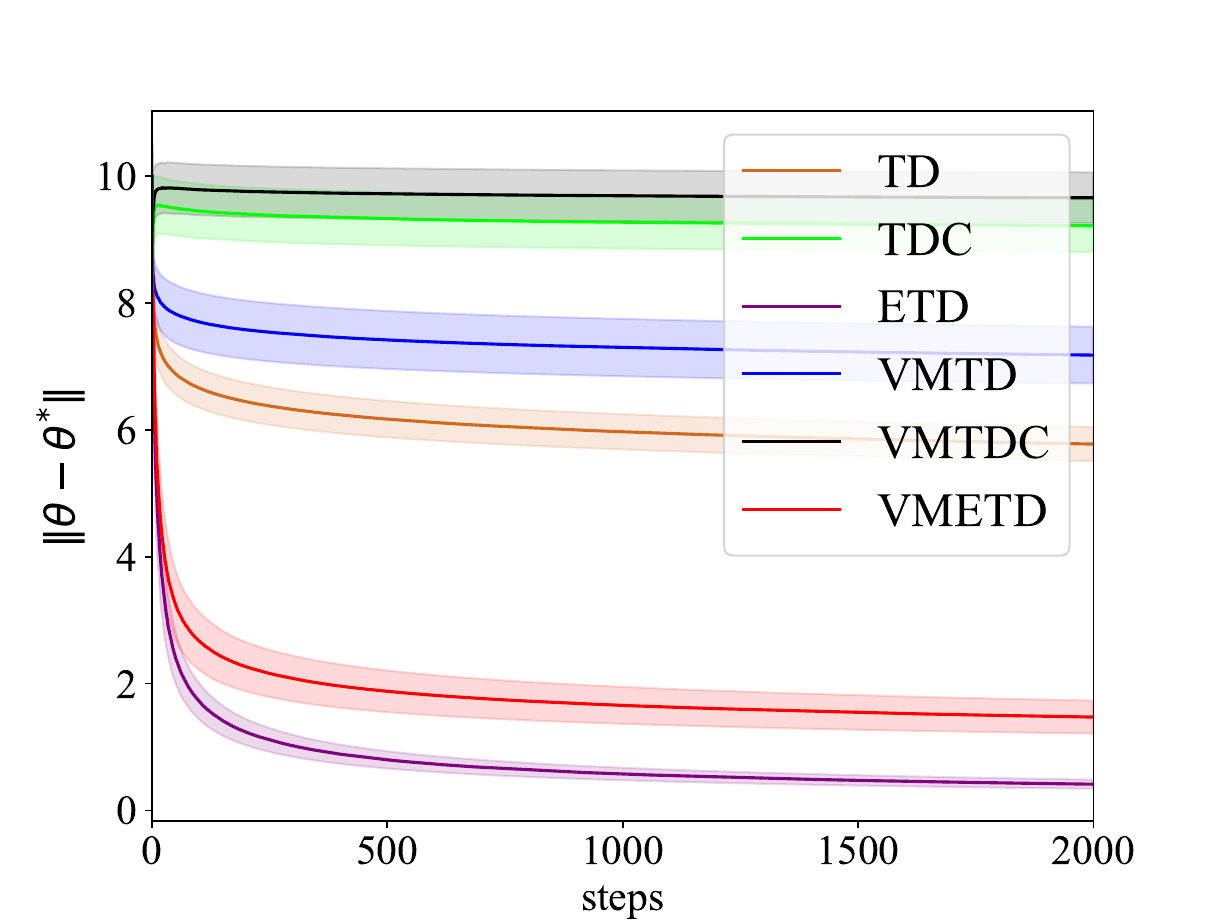}
      \label{2-state}
     }
     \subfigure[off-policy 2-state]{
      \includegraphics[width=0.65\columnwidth, height=0.58\columnwidth]{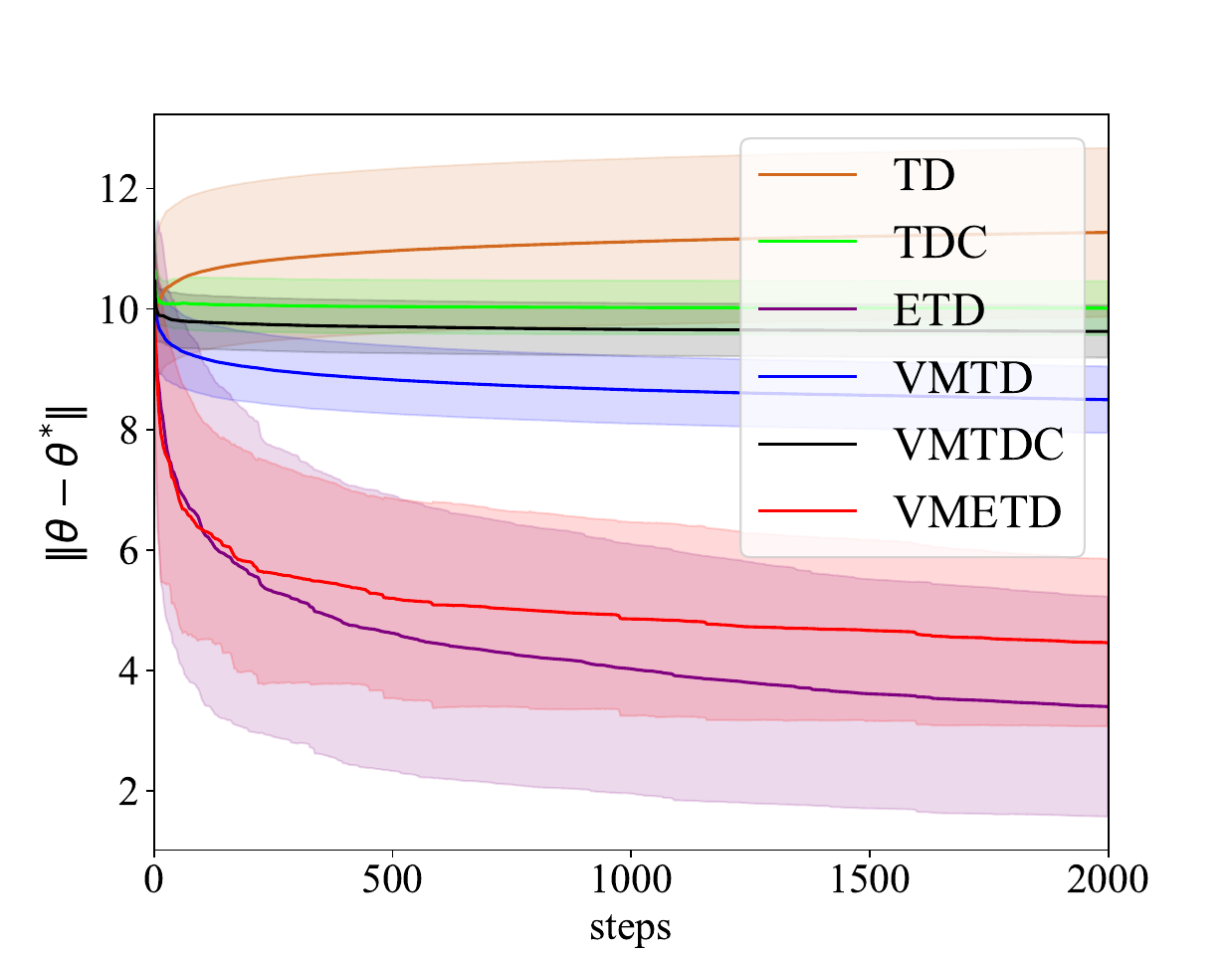}
      \label{7-state}
     }
     \subfigure[Maze]{
      \includegraphics[width=0.65\columnwidth, height=0.58\columnwidth]{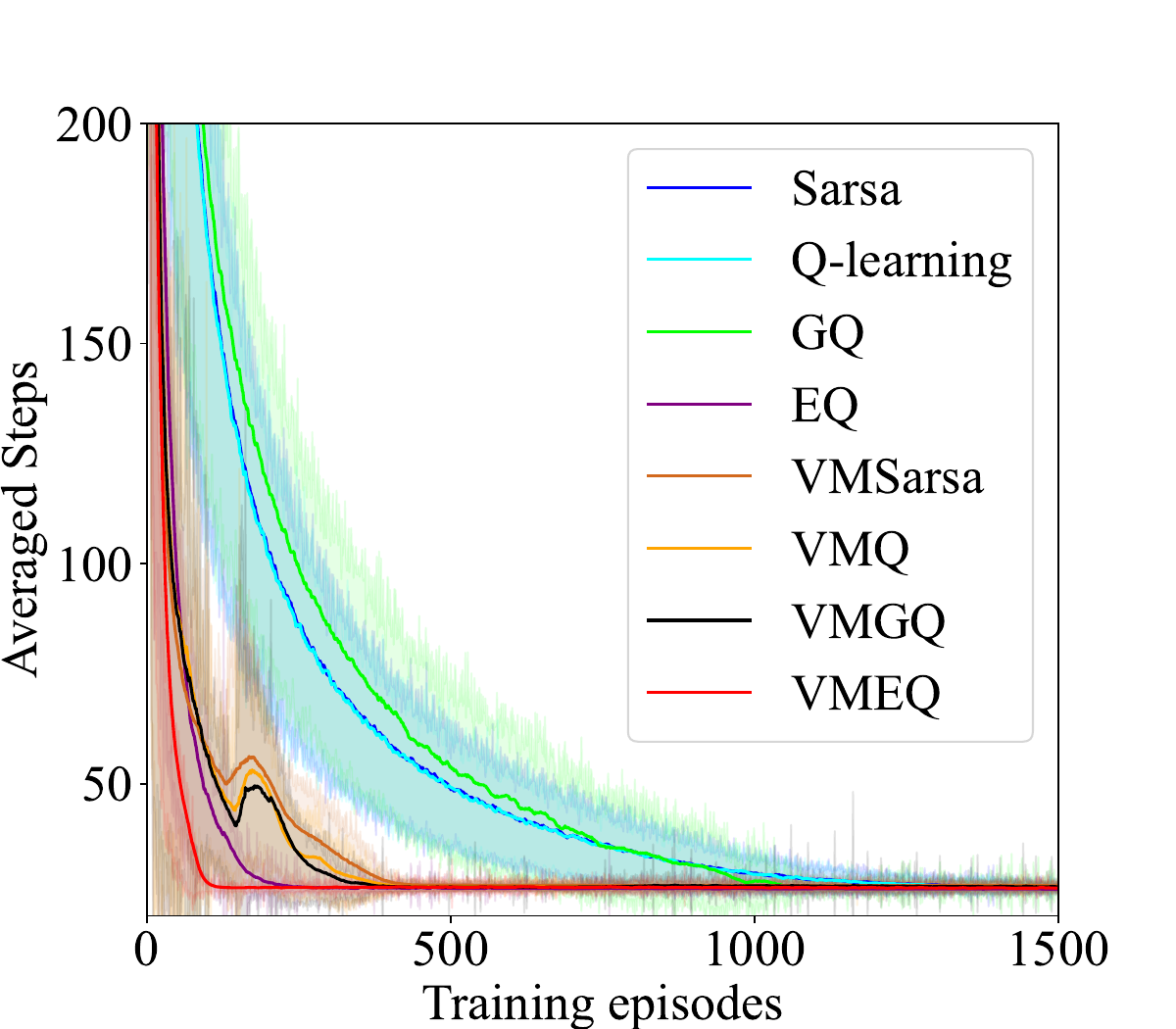}
      \label{MazeFull}
     }\\
     \subfigure[Cliff Walking]{
      \includegraphics[width=0.65\columnwidth, height=0.58\columnwidth]{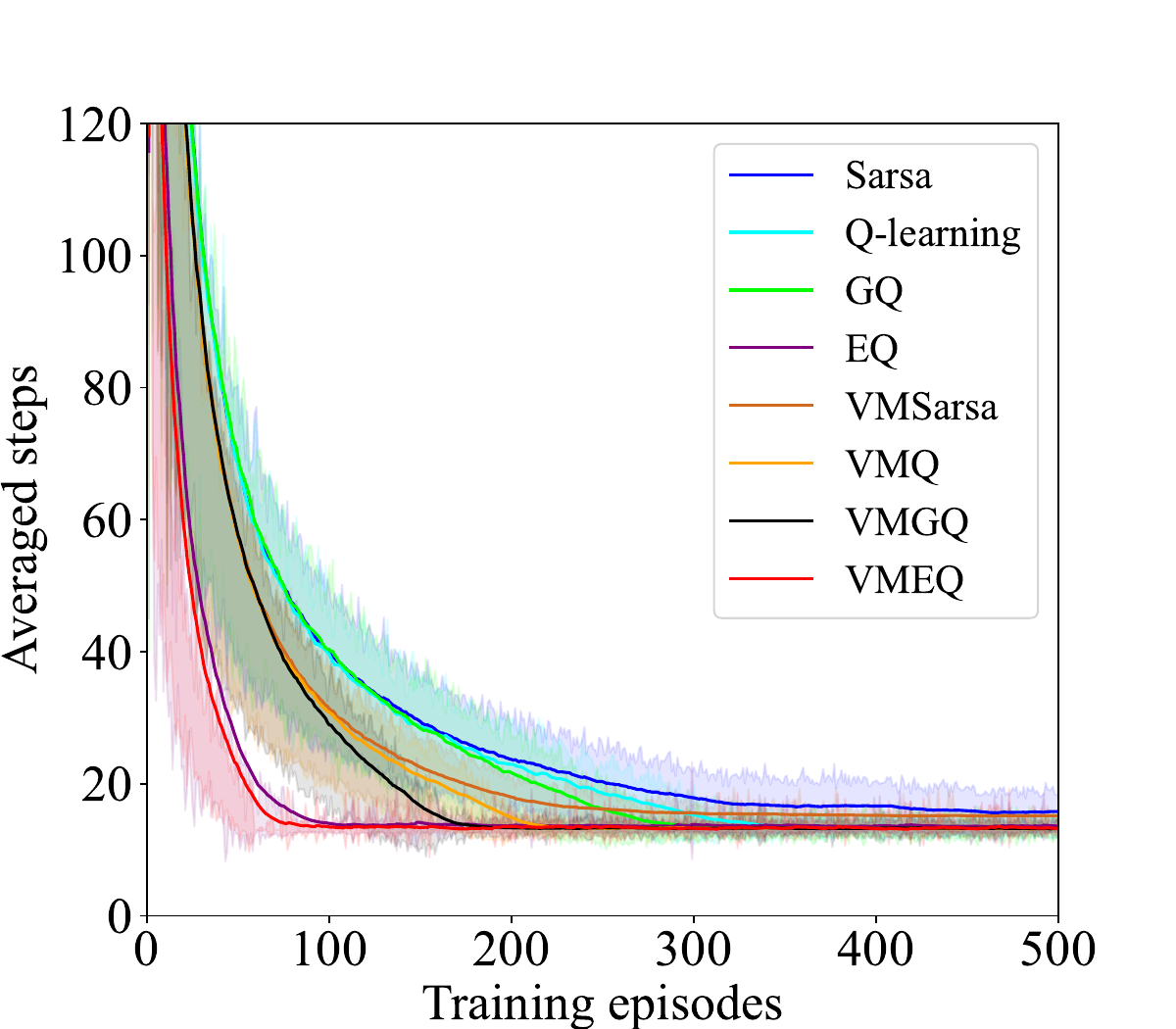}
      \label{CliffWalkingFull}
     }
     \subfigure[Mountain Car]{
      \includegraphics[width=0.65\columnwidth, height=0.58\columnwidth]{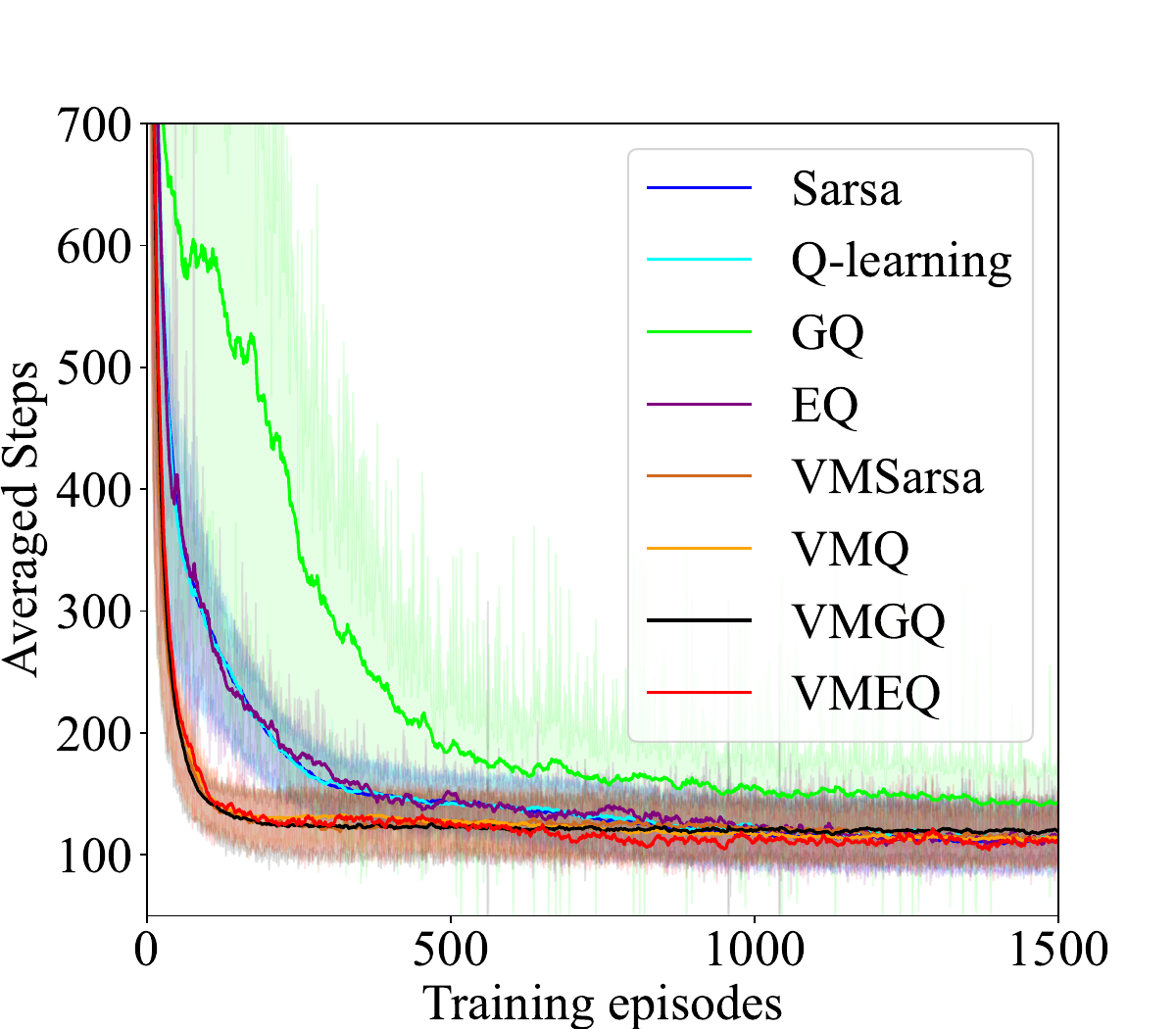}
      \label{MountainCarFull}
     }
     \subfigure[Acrobot]{
      \includegraphics[width=0.65\columnwidth, height=0.58\columnwidth]{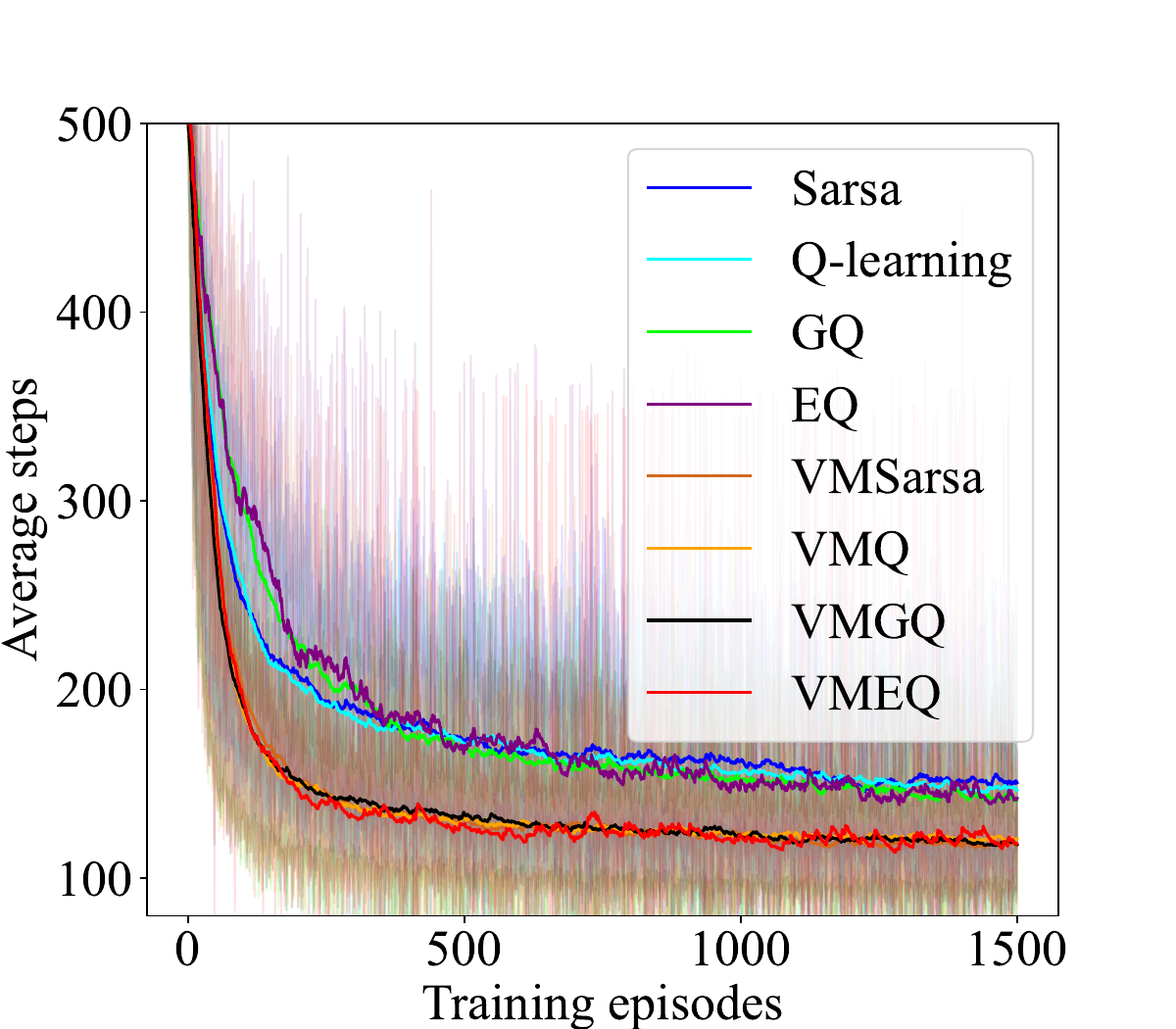}
      \label{AcrobotFull}
     }
      \caption{Learning curses of one evaluation environment and four control environments.}
      \label{Complete_full}
     \end{center}
     \vskip -0.2in
      \end{figure*}
    \section{Experimental Studies}
    This section assesses algorithm performance through experiments, 
    which are divided into policy evaluation experiments and control experiments.
    The evaluation experimental environment is the 2-state. 
    In a 2-state environment, we conducted two types of experiments—on-policy 
    and off-policy—to verify the relationship between the convergence speed of 
    the algorithm and the smallest eigenvalue of the key matrix $(\textbf{A}+\textbf{A}^{\top})/2$.
    Control experiments, by allowing the algorithm to interact 
    with the environment to optimize the policy, can evaluate its 
    performance in learning the optimal policy. This provides a more 
    comprehensive assessment of the algorithm's overall capabilities.
%     The control experimental environments are Maze, CliffWalking-v0, MountainCar-v0, and Acrobot-v1.
    To verify the optimal policy invariance of variance minimization, we needed to eliminate the 
    influence of function approximation. Therefore, we introduced control experiments with tabular 
    value functions in the Maze and CliffWalking environments. 
    To further observe the effects of variance minimization, we introduced experiments with function approximation in the Mountain Car and Acrobot environments.
    The control algorithms for TDC, ETD, VMTDC, and VMETD are named GQ, EQ, VMGQ, and VMEQ, respectively.
    For TD and VMTD control algorithms, there are two variants each: Sarsa and Q-learning for TD, and VMSarsa and VMQ for VMTD.
    
    \subsection{Testing Tasks}
    
    \textbf{Maze}:  The learning agent should find the shortest path from the upper
    left corner to the lower right corner. 
    \begin{minipage}{0.2\textwidth}
     In each state,
     there are four alternative actions: $up$, $down$, $left$, and $right$, which
     takes the agent deterministically to the corresponding neighbor state,
     except when a movement is blocked by an obstacle or the edge
     of the maze. 
    \end{minipage}
    \begin{minipage}{0.25\textwidth}
      \centering
      \includegraphics[scale=0.25]{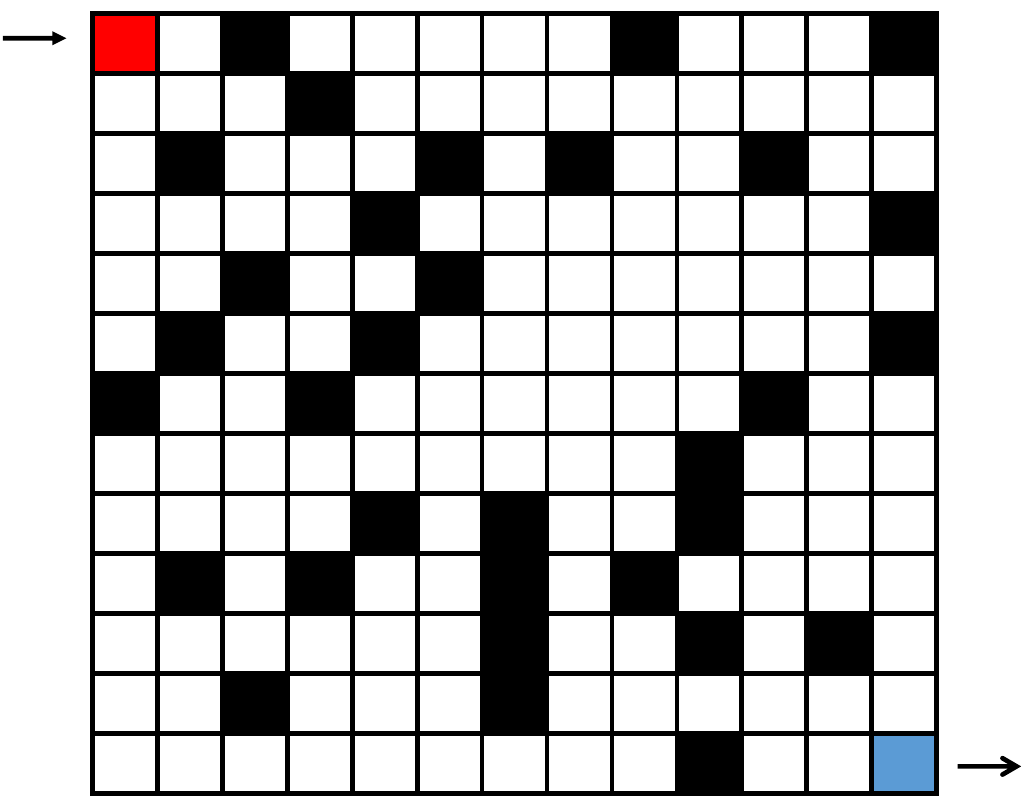}
    \end{minipage}
    Rewards are $-1$ in all transitions until the
    agent reaches the goal state.
    The discount factor $\gamma=0.99$, and states $s$ are represented by tabular
    features. The maximum number of moves in the game is set to 1000.
    
    \textbf{The other three control environments}: Cliff Walking, Mountain Car, and Acrobot are 
    selected from the gym's official website and correspond to the following 
    versions: ``CliffWalking-v0'', ``MountainCar-v0'' and ``Acrobot-v1''. 
    For specific details, please refer to the gym's official website.
    The maximum number of steps for the MountainCar environment is set to 1000, 
    while the default settings are used for the other two environments. In  MountainCar and Acrobot, features are generated by tile coding.
    
    For all policy evaluation experiments, each experiment 
    is independently run 100 times.
    For all control experiments, each experiment is independently run 50 times.
    For specific experimental parameters, please refer to the appendix.
    
    \subsection{Experimental Results and Analysis}
    Figure \ref{2-state} shows the learning curves for the on-policy 
    2-state policy evaluation experiment. In this setup, 
    the convergence speed of ETD, VMETD, TD, VMTD, TDC, and VMTDC decreases 
    sequentially. Table \ref{tab:min_eigenvalues} indicates that the smallest eigenvalue 
    of the key matrix for these four algorithms is greater than 0 
    and decreases sequentially, which is consistent with the 
    experimental curves and table values.
    
    Figure \ref{7-state} displays the learning curves for the off-policy 
    2-state policy evaluation experiment. 
    The convergence speed of ETD, VMETD, VMTD, VMTDC, and 
    TDC decreases sequentially, while TD diverges. Table \ref{tab:min_eigenvalues} 
    shows that the smallest eigenvalue of the key matrix for 
    ETD, VMETD, VMTD, VMTDC, and TDC are greater than 0 and 
    decreases sequentially, while the smallest eigenvalue 
    for TD is less than 0. This is consistent with the 
    experimental curves and table values. Although VMTD is guaranteed to converge under 
    on-policy conditions, it still converges in the 
    off-policy 2-state scenario. The update formula 
    of VMTD indicates that it is essentially an 
    adjustment and correction of the TD update, 
    with the introduction of the parameter $\omega$ 
    making the variance of the gradient estimate 
    more stable, thereby making the update of theta more stable.
    
    Figures \ref{MazeFull}, \ref{CliffWalkingFull}, \ref{MountainCarFull} and \ref{AcrobotFull} show the learning curves 
    for control experiments. 
    As shown in Figures \ref{MazeFull} and \ref{CliffWalkingFull}, all algorithms converge to the optimal policy, indicating that the VM algorithm satisfies the optimal policy invariance.
    
    A common feature 
    observed across these experiments is that VMEQ 
    outperforms EQ, VMGQ outperforms GQ, VMQ outperforms 
    Q-learning and VMSarsa outperforms Sarsa. For the 
    Maze and CliffWalking experiments, VMEQ demonstrated 
    the best performance with the fastest convergence speed. 
    In the MountainCar and Acrobot experiments, the performance 
    of the four VM algorithms was nearly identical and all 
    outperformed the other algorithms.
    
%     Overall, whether in policy evaluation experiments or 
%     control experiments, the VM algorithms have 
%     demonstrated superior performance, 
%     especially excelling in the control experiments.
\section{Conclusion and Future Work}
Value-based RL typically aims 
to minimize error as an optimization objective. 
As an alternation, this study proposes two new objective 
functions: VBE and VPBE, and derives an on-policy algorithm: 
VMTD and two off-policy algorithms: VMTDC and VMETD. 
All algorithms demonstrated superior performance in policy 
evaluation and control experiments.
Both algorithms demonstrated superior performance in policy 
evaluation and control experiments.
Future work include
(1) apply the VM approach to temporal difference learning algorithms, e.g., HTD, Vtrace, Proximal GTD2, TDRC, Tree Backup, and ABTD \cite{ghiassian2024off}.
(2) extensions of VBE and VPBE to multi-step returns. 
(3) extensions to nonlinear approximations.

\bibliography{aaai25}
\onecolumn
\appendix
\section{Relevant proofs}
\subsection{Proof of Theorem 2}
\label{proofth1}
\begin{proof}
\label{th1proof}   
 The proof is based on Borkar's Theorem for
 general stochastic approximation recursions with two time scales \cite{borkar1997stochastic}. 
 A new one-step linear TD solution is defined as: 
\begin{equation*}
0=\mathbb{E}[(\delta-\mathbb{E}[\delta]) \phi]=-A\theta+b.
\end{equation*}
Thus, the VMTD's solution is $\theta_{\text{VMTD}}=A^{-1}b$. First, note that recursion (\ref{theta}) can be rewritten as
\begin{equation*}
\theta_{k+1}\leftarrow \theta_k+\beta_k\xi(k),
\end{equation*}
where
\begin{equation*}
\xi(k)=\frac{\alpha_k}{\beta_k}(\delta_k-\omega_k)\phi_k
\end{equation*}
Due to the settings of step-size schedule $\alpha_k = o(\beta_k)$,
$\xi(k)\rightarrow 0$ almost surely as $k\rightarrow\infty$. 
 That is the increments in iteration (\ref{omega}) are uniformly larger than
 those in (\ref{theta}), thus (\ref{omega}) is the faster recursion.
 Along the faster time scale, iterations of (\ref{omega}) and (\ref{theta})
 are associated with the ODEs system as follows:
\begin{equation}
 \dot{\theta}(t) = 0,
\label{thetaFast}
\end{equation}
\begin{equation}
 \dot{\omega}(t)=\mathbb{E}[\delta_t|\theta(t)]-\omega(t).
\label{omegaFast}
\end{equation}
 Based on the ODE (\ref{thetaFast}), $\theta(t)\equiv \theta$ when
 viewed from the faster timescale. 
 By the Hirsch lemma \cite{hirsch1989convergent}, it follows that
$||\theta_k-\theta||\rightarrow 0$ a.s. as $k\rightarrow \infty$ for some
$\theta$ that depends on the initial condition $\theta_0$ of recursion
 (\ref{theta}).
 Thus, the ODE pair (\ref{thetaFast})-(\ref{omegaFast}) can be written as
\begin{equation}
 \dot{\omega}(t)=\mathbb{E}[\delta_t|\theta]-\omega(t).
\label{omegaFastFinal}
\end{equation}
 Consider the function $h(\omega)=\mathbb{E}[\delta|\theta]-\omega$,
 i.e., the driving vector field of the ODE (\ref{omegaFastFinal}).
 It is easy to find that the function $h$ is Lipschitz with coefficient
$-1$.
 Let $h_{\infty}(\cdot)$ be the function defined by
$h_{\infty}(\omega)=\lim_{x\rightarrow \infty}\frac{h(x\omega)}{x}$.
 Then  $h_{\infty}(\omega)= -\omega$,  is well-defined. 
 For (\ref{omegaFastFinal}), $\omega^*=\mathbb{E}[\delta|\theta]$
 is the unique globally asymptotically stable equilibrium.
 For the ODE
  \begin{equation}
 \dot{\omega}(t) = h_{\infty}(\omega(t))= -\omega(t),
 \label{omegaInfty}
 \end{equation}
 apply $\vec{V}(\omega)=(-\omega)^{\top}(-\omega)/2$ as its
 associated strict Liapunov function. Then,
 the origin of (\ref{omegaInfty}) is a globally asymptotically stable
 equilibrium. Consider now the recursion (\ref{omega}).
 Let $M_{k+1}=(\delta_k-\omega_k)
 -\mathbb{E}[(\delta_k-\omega_k)|\mathcal{F}(k)]$,
 where $\mathcal{F}(k)=\sigma(\omega_l,\theta_l,l\leq k;\phi_s,\phi_s',r_s,s<k)$, $k\geq 1$ are the sigma fields
 generated by $\omega_0,\theta_0,\omega_{l+1},\theta_{l+1},\phi_l,\phi_l'$, $0\leq l<k$.
 It is easy to verify that $M_{k+1},k\geq0$ are integrable random variables that 
 satisfy $\mathbb{E}[M_{k+1}|\mathcal{F}(k)]=0$, $\forall k\geq0$.
 Because $\phi_k$, $r_k$, and $\phi_k'$ have
 uniformly bounded second moments, it can be seen that for some constant $c_1>0$, $\forall k\geq0$,
\begin{equation*}
 \mathbb{E}[||M_{k+1}||^2|\mathcal{F}(k)]\leq
 c_1(1+||\omega_k||^2+||\theta_k||^2).
\end{equation*}
Now Assumptions (A1) and (A2) of \cite{borkar2000ode} are verified.
 Furthermore, Assumptions (TS) of \cite{borkar2000ode} are satisfied by our
 conditions on the step-size sequences $\alpha_k$, $\beta_k$. Thus,
 by Theorem 2.2 of \cite{borkar2000ode} we obtain that $||\omega_k-\omega^*||\rightarrow 0$ almost surely as $k\rightarrow \infty$.
Consider now the slower time scale recursion (\ref{theta}).
 Based on the above analysis, (\ref{theta}) can be rewritten as 
\begin{equation*}
\theta_{k+1}\leftarrow
\theta_{k}+\alpha_k(\delta_k-\mathbb{E}[\delta_k|\theta_k])\phi_k.
\end{equation*}
Let $\mathcal{G}(k)=\sigma(\theta_l,l\leq k;\phi_s,\phi_s',r_s,s<k)$, 
$k\geq 1$ be the sigma fields
 generated by $\theta_0,\theta_{l+1},\phi_l,\phi_l'$,
$0\leq l<k$.
 Let $Z_{k+1} = Y_{t}-\mathbb{E}[Y_{t}|\mathcal{G}(k)]$,
 where
\begin{equation*}
 Y_{k}=(\delta_k-\mathbb{E}[\delta_k|\theta_k])\phi_k.
\end{equation*}
 Consequently,
\begin{equation*}
\begin{array}{ccl}
 \mathbb{E}[Y_t|\mathcal{G}(k)]&=&\mathbb{E}[(\delta_k-\mathbb{E}[\delta_k|\theta_k])\phi_k|\mathcal{G}(k)]\\
&=&\mathbb{E}[\delta_k\phi_k|\theta_k]
 -\mathbb{E}[\mathbb{E}[\delta_k|\theta_k]\phi_k]\\
&=&\mathbb{E}[\delta_k\phi_k|\theta_k]
 -\mathbb{E}[\delta_k|\theta_k]\mathbb{E}[\phi_k]\\
&=&\mathrm{Cov}(\delta_k|\theta_k,\phi_k),
\end{array}
\end{equation*}
 where $\mathrm{Cov}(\cdot,\cdot)$ is a covariance operator.
Thus,
 \begin{equation*}
\begin{array}{ccl}
 Z_{k+1}&=&(\delta_k-\mathbb{E}[\delta_k|\theta_k])\phi_k-\mathrm{Cov}(\delta_k|\theta_k,\phi_k).
\end{array}
\end{equation*}
 It is easy to verify that $Z_{k+1},k\geq0$ are integrable random variables that 
 satisfy $\mathbb{E}[Z_{k+1}|\mathcal{G}(k)]=0$, $\forall k\geq0$.
 Also, because $\phi_k$, $r_k$, and $\phi_k'$ have
 uniformly bounded second moments, it can be seen that for some constant
$c_2>0$, $\forall k\geq0$,
\begin{equation*}
 \mathbb{E}[||Z_{k+1}||^2|\mathcal{G}(k)]\leq
 c_2(1+||\theta_k||^2).
\end{equation*}

 Consider now the following ODE associated with (\ref{theta}):
\begin{equation}
\begin{array}{ccl}
 \dot{\theta}(t)&=&\mathrm{Cov}(\delta|\theta(t),\phi)\\
&=&\mathrm{Cov}(r+(\gamma\phi'-\phi)^{\top}\theta(t),\phi)\\
&=&\mathrm{Cov}(r,\phi)-\mathrm{Cov}(\theta(t)^{\top}(\phi-\gamma\phi'),\phi)\\
&=&\mathrm{Cov}(r,\phi)-\theta(t)^{\top}\mathrm{Cov}(\phi-\gamma\phi',\phi)\\
&=&\mathrm{Cov}(r,\phi)-\mathrm{Cov}(\phi-\gamma\phi',\phi)^{\top}\theta(t)\\
&=&\mathrm{Cov}(r,\phi)-\mathrm{Cov}(\phi,\phi-\gamma\phi')\theta(t)\\
&=&-A\theta(t)+b.
\end{array}
\label{odetheta}
\end{equation}
 Let $\vec{h}(\theta(t))$ be the driving vector field of the ODE
 (\ref{odetheta}).
\begin{equation*}
 \vec{h}(\theta(t))=-A\theta(t)+b.
\end{equation*}
 Consider the cross-covariance matrix,
\begin{equation}
\begin{array}{ccl}
 A &=& \mathrm{Cov}(\phi,\phi-\gamma\phi')\\
  &=&\frac{\mathrm{Cov}(\phi,\phi)+\mathrm{Cov}(\phi-\gamma\phi',\phi-\gamma\phi')-\mathrm{Cov}(\gamma\phi',\gamma\phi')}{2}\\
  &=&\frac{\mathrm{Cov}(\phi,\phi)+\mathrm{Cov}(\phi-\gamma\phi',\phi-\gamma\phi')-\gamma^2\mathrm{Cov}(\phi',\phi')}{2}\\
  &=&\frac{(1-\gamma^2)\mathrm{Cov}(\phi,\phi)+\mathrm{Cov}(\phi-\gamma\phi',\phi-\gamma\phi')}{2},\\
\end{array}
\label{covariance}
\end{equation}
 where we eventually used $\mathrm{Cov}(\phi',\phi')=\mathrm{Cov}(\phi,\phi)$
\footnote{The covariance matrix $\mathrm{Cov}(\phi',\phi')$ is equal to
 the covariance matrix $\mathrm{Cov}(\phi,\phi)$ if the initial state is re-reachable or
 initialized randomly in a Markov chain for on-policy update.}.
 Note that the covariance matrix $\mathrm{Cov}(\phi,\phi)$ and
$\mathrm{Cov}(\phi-\gamma\phi',\phi-\gamma\phi')$ are semi-positive
 definite. Then, the matrix $A$ is semi-positive definite because  $A$ is
 linearly combined  by  two positive-weighted semi-positive definite matrice
 (\ref{covariance}).
 Furthermore, $A$ is nonsingular due to the assumption.
 Hence, the cross-covariance matrix $A$ is positive definite.

 Therefore,
$\theta^*=A^{-1}b$ can be seen to be the unique globally asymptotically
 stable equilibrium for ODE (\ref{odetheta}).
 Let $\vec{h}_{\infty}(\theta)=\lim_{r\rightarrow
\infty}\frac{\vec{h}(r\theta)}{r}$. Then
$\vec{h}_{\infty}(\theta)=-A\theta$ is well-defined. 
 Consider now the ODE
\begin{equation}
 \dot{\theta}(t)=-A\theta(t).
\label{odethetafinal}
\end{equation}
 The ODE (\ref{odethetafinal}) has the origin of its unique globally asymptotically stable equilibrium.
 Thus, the assumption (A1) and (A2) are verified.
    \end{proof}

\subsection{Proof of Theorem 3}
\label{proofth2}
\begin{proof}
The proof is similar to that given by \cite{sutton2009fast} for TDC, but it is based on multi-time-scale stochastic approximation.

For the VMTDC algorithm, a new one-step linear TD solution is defined as:
\begin{equation*}
    0=\mathbb{E}[({\phi} - \gamma {\phi}' - \mathbb{E}[{\phi} - \gamma {\phi}']){\phi}^\top]\mathbb{E}[{\phi} {\phi}^{\top}]^{-1}\mathbb{E}[(\delta -\mathbb{E}[\delta]){\phi}]=\textbf{A}^{\top}\textbf{C}^{-1}(-\textbf{A}{\theta}+{b}).
\end{equation*}
The matrix $\textbf{A}^{\top}\textbf{C}^{-1}\textbf{A}$ is positive definite. Thus, the  VMTD's solution is
${\theta}_{\text{VMTDC}}=\textbf{A}^{-1}{b}$.

First, note that recursion (\ref{thetavmtdc}) and (\ref{uvmtdc}) can be rewritten as, respectively, 
\begin{equation*}
 {\theta}_{k+1}\leftarrow {\theta}_k+\zeta_k {x}(k),
\end{equation*}
\begin{equation*}
 {u}_{k+1}\leftarrow {u}_k+\beta_k {y}(k),
\end{equation*}
where 
\begin{equation*}
 {x}(k)=\frac{\alpha_k}{\zeta_k}[(\delta_{k}- \omega_k) {\phi}_k - \gamma{\phi}'_{k}({\phi}^{\top}_k {u}_k)],
\end{equation*}
\begin{equation*}
 {y}(k)=\frac{\zeta_k}{\beta_k}[\delta_{k}-\omega_k - {\phi}^{\top}_k {u}_k]{\phi}_k.
\end{equation*}

Recursion (\ref{thetavmtdc}) can also be rewritten as
\begin{equation*}
 {\theta}_{k+1}\leftarrow {\theta}_k+\beta_k z(k),
\end{equation*}
where
\begin{equation*}
 z(k)=\frac{\alpha_k}{\beta_k}[(\delta_{k}- \omega_k) {\phi}_k - \gamma{\phi}'_{k}({\phi}^{\top}_k {u}_k)],
\end{equation*}

Due to the settings of the step-size schedule 
$\alpha_k = o(\zeta_k)$, $\zeta_k = o(\beta_k)$, ${x}(k)\rightarrow 0$, ${y}(k)\rightarrow 0$, $z(k)\rightarrow 0$ almost surely as $k\rightarrow 0$.
That is the increments in iteration (\ref{omegavmtdc}) are uniformly larger than
those in (\ref{uvmtdc}) and  the increments in iteration (\ref{uvmtdc}) are uniformly larger than
those in (\ref{thetavmtdc}), thus (\ref{omegavmtdc}) is the fastest recursion, (\ref{uvmtdc}) is the second fast recursion, and (\ref{thetavmtdc}) is the slower recursion.
Along the fastest time scale, iterations of (\ref{thetavmtdc}), (\ref{uvmtdc}) and (\ref{omegavmtdc})
are associated with the ODEs system as follows:
\begin{equation}
 \dot{{\theta}}(t) = 0,
    \label{thetavmtdcFastest}
\end{equation}
\begin{equation}
 \dot{{u}}(t) = 0,
    \label{uvmtdcFastest}
\end{equation}
\begin{equation}
 \dot{\omega}(t)=\mathbb{E}[\delta_t|{u}(t),{\theta}(t)]-\omega(t).
    \label{omegavmtdcFastest}
\end{equation}

Based on the ODE (\ref{thetavmtdcFastest}) and (\ref{uvmtdcFastest}), both ${\theta}(t)\equiv {\theta}$
and ${u}(t)\equiv {u}$ when viewed from the fastest timescale.
By the Hirsch lemma \cite{hirsch1989convergent}, it follows that
$||{\theta}_k-{\theta}||\rightarrow 0$ a.s. as $k\rightarrow \infty$ for some
${\theta}$ that depends on the initial condition ${\theta}_0$ of recursion
(\ref{thetavmtdc}) and $||{u}_k-{u}||\rightarrow 0$ a.s. as $k\rightarrow \infty$ for some
$u$ that depends on the initial condition $u_0$ of recursion
(\ref{uvmtdc}). Thus, the ODE pair (\ref{thetavmtdcFastest})-(ref{omegavmtdcFastest})
can be written as 
\begin{equation}
 \dot{\omega}(t)=\mathbb{E}[\delta_t|{u},{\theta}]-\omega(t).
    \label{omegavmtdcFastestFinal}
\end{equation}

Consider the function $h(\omega)=\mathbb{E}[\delta|{\theta},{u}]-\omega$,
i.e., the driving vector field of the ODE (\ref{omegavmtdcFastestFinal}).
It is easy to find that the function $h$ is Lipschitz with coefficient
$-1$.
Let $h_{\infty}(\cdot)$ be the function defined by
 $h_{\infty}(\omega)=\lim_{r\rightarrow \infty}\frac{h(r\omega)}{r}$.
 Then  $h_{\infty}(\omega)= -\omega$,  is well-defined. 
 For (\ref{omegavmtdcFastestFinal}), $\omega^*=\mathbb{E}[\delta|{\theta},{u}]$
is the unique globally asymptotically stable equilibrium.
For the ODE
\begin{equation}
 \dot{\omega}(t) = h_{\infty}(\omega(t))= -\omega(t),
 \label{omegavmtdcInfty}
\end{equation}
apply $\vec{V}(\omega)=(-\omega)^{\top}(-\omega)/2$ as its
associated strict Liapunov function. Then,
the origin of (\ref{omegavmtdcInfty}) is a globally asymptotically stable
equilibrium.

Consider now the recursion (\ref{omegavmtdc}).
Let
$M_{k+1}=(\delta_k-\omega_k)
-\mathbb{E}[(\delta_k-\omega_k)|\mathcal{F}(k)]$,
where $\mathcal{F}(k)=\sigma(\omega_l,{u}_l,{\theta}_l,l\leq k;{\phi}_s,{\phi}_s',r_s,s<k)$, 
$k\geq 1$ are the sigma fields
generated by $\omega_0,u_0,{\theta}_0,\omega_{l+1},{u}_{l+1},{\theta}_{l+1},{\phi}_l,{\phi}_l'$,
$0\leq l<k$.
It is easy to verify that $M_{k+1},k\geq0$ are integrable random variables that 
satisfy $\mathbb{E}[M_{k+1}|\mathcal{F}(k)]=0$, $\forall k\geq0$.
Because ${\phi}_k$, $r_k$, and ${\phi}_k'$ have
uniformly bounded second moments, it can be seen that for some constant
$c_1>0$, $\forall k\geq0$,
\begin{equation*}
\mathbb{E}[||M_{k+1}||^2|\mathcal{F}(k)]\leq
c_1(1+||\omega_k||^2+||{u}_k||^2+||{\theta}_k||^2).
\end{equation*}

Now Assumptions (A1) and (A2) of \cite{borkar2000ode} are verified.
Furthermore, Assumptions (TS) of \cite{borkar2000ode} is satisfied by our
conditions on the step-size sequences $\alpha_k$,$\zeta_k$, $\beta_k$. Thus,
by Theorem 2.2 of \cite{borkar2000ode} we obtain that
$||\omega_k-\omega^*||\rightarrow 0$ almost surely as $k\rightarrow \infty$.

Consider now the second time scale recursion (\ref{uvmtdc}).
Based on the above analysis, (\ref{uvmtdc}) can be rewritten as
% \begin{equation*}
%     {u}_{k+1}\leftarrow u_{k}+\zeta_{k}[\delta_{k}-\mathbb{E}[\delta_k|{u}_k,{\theta}_k] - {\phi}^{\top} (s_k) {u}_k]{\phi}(s_k).
% \end{equation*}
\begin{equation}
 \dot{{\theta}}(t) = 0,
    \label{thetavmtdcFaster}
\end{equation}
\begin{equation}
 \dot{u}(t) = \mathbb{E}[(\delta_t-\mathbb{E}[\delta_t|{u}(t),{\theta}(t)]){\phi}_t|{\theta}(t)] - \textbf{C}{u}(t).
    \label{uvmtdcFaster}
\end{equation}
The ODE (\ref{thetavmtdcFaster}) suggests that ${\theta}(t)\equiv {\theta}$ (i.e., a time-invariant parameter)
when viewed from the second fast timescale.
By the Hirsch lemma \cite{hirsch1989convergent}, it follows that
$||{\theta}_k-{\theta}||\rightarrow 0$ a.s. as $k\rightarrow \infty$ for some
${\theta}$ that depends on the initial condition ${\theta}_0$ of recursion
(\ref{thetavmtdc}). 

Consider now the recursion (\ref{uvmtdc}).
Let
$N_{k+1}=((\delta_k-\mathbb{E}[\delta_k]) - {\phi}_k {\phi}^{\top}_k {u}_k) -\mathbb{E}[((\delta_k-\mathbb{E}[\delta_k]) - {\phi}_k {\phi}^{\top}_k {u}_k)|\mathcal{I} (k)]$,
where $\mathcal{I}(k)=\sigma({u}_l,{\theta}_l,l\leq k;{\phi}_s,{\phi}_s',r_s,s<k)$, 
$k\geq 1$ are the sigma fields
generated by ${u}_0,{\theta}_0,{u}_{l+1},{\theta}_{l+1},{\phi}_l,{\phi}_l'$,
$0\leq l<k$.
It is easy to verify that $N_{k+1},k\geq0$ are integrable random variables that 
satisfy $\mathbb{E}[N_{k+1}|\mathcal{I}(k)]=0$, $\forall k\geq0$.
Because ${\phi}_k$, $r_k$, and ${\phi}_k'$ have
uniformly bounded second moments, it can be seen that for some constant
$c_2>0$, $\forall k\geq0$,
\begin{equation*}
\mathbb{E}[||N_{k+1}||^2|\mathcal{I}(k)]\leq
c_2(1+||{u}_k||^2+||{\theta}_k||^2).
\end{equation*}

Because ${\theta}(t)\equiv {\theta}$ from (\ref{thetavmtdcFaster}), the ODE pair (\ref{thetavmtdcFaster})-(\ref{uvmtdcFaster})
can be written as 
\begin{equation}
 \dot{{u}}(t) = \mathbb{E}[(\delta_t-\mathbb{E}[\delta_t|{\theta}]){\phi}_t|{\theta}] - \textbf{C}{u}(t).
    \label{uvmtdcFasterFinal}
\end{equation}
Now consider the function $h({u})=\mathbb{E}[\delta_t-\mathbb{E}[\delta_t|{\theta}]|{\theta}] -\textbf{C}{u}$, i.e., the
driving vector field of the ODE (\ref{uvmtdcFasterFinal}). For (\ref{uvmtdcFasterFinal}),
${u}^* = \textbf{C}^{-1}\mathbb{E}[(\delta-\mathbb{E}[\delta|{\theta}]){\phi}|{\theta}]$ is the unique globally asymptotically
stable equilibrium. Let $h_{\infty}({u})=-\textbf{C}{u}$.
For the ODE
\begin{equation}
 \dot{{u}}(t) = h_{\infty}({u}(t))= -\textbf{C}{u}(t),
    \label{uvmtdcInfty}
\end{equation}
the origin of (\ref{uvmtdcInfty}) is a globally asymptotically stable
equilibrium because $\textbf{C}$ is a positive definite matrix (because it is nonnegative definite and nonsingular).
Now Assumptions (A1) and (A2) of \cite{borkar2000ode} are verified.
Furthermore, Assumptions (TS) of \cite{borkar2000ode} is satisfied by our
conditions on the step-size sequences $\alpha_k$,$\zeta_k$, $\beta_k$. Thus,
by Theorem 2.2 of \cite{borkar2000ode} we obtain that
$||{u}_k-{u}^*||\rightarrow 0$ almost surely as $k\rightarrow \infty$.

Consider now the slower timescale recursion (\ref{thetavmtdc}). In the light of the above,
(\ref{thetavmtdc}) can be rewritten as 
\begin{equation}
 {\theta}_{k+1} \leftarrow {\theta}_{k} + \alpha_k (\delta_k -\mathbb{E}[\delta_k|{\theta}_k]) {\phi}_k\\
 - \alpha_k \gamma{\phi}'_{k}({\phi}^{\top}_k \textbf{C}^{-1}\mathbb{E}[(\delta_k -\mathbb{E}[\delta_k|{\theta}_k]){\phi}|{\theta}_k]).
\end{equation}
Let $\mathcal{G}(k)=\sigma({\theta}_l,l\leq k;{\phi}_s,{\phi}_s',r_s,s<k)$, 
$k\geq 1$ be the sigma fields
generated by ${\theta}_0,{\theta}_{l+1},{\phi}_l,{\phi}_l'$,
$0\leq l<k$. Let
\begin{equation*}
    \begin{array}{ccl}
 Z_{k+1}&=&((\delta_k -\mathbb{E}[\delta_k|{\theta}_k]) {\phi}_k - \gamma {\phi}'_{k}{\phi}^{\top}_k \textbf{C}^{-1}\mathbb{E}[(\delta_k -\mathbb{E}[\delta_k|{\theta}_k]){\phi}|{\theta}_k])\\ 
     & &-\mathbb{E}[((\delta_k -\mathbb{E}[\delta_k|{\theta}_k]) {\phi}_k - \gamma {\phi}'_{k}{\phi}^{\top}_k \textbf{C}^{-1}\mathbb{E}[(\delta_k -\mathbb{E}[\delta_k|{\theta}_k]){\phi}|{\theta}_k])|\mathcal{G}(k)]\\
    &=&((\delta_k -\mathbb{E}[\delta_k|{\theta}_k]) {\phi}_k - \gamma {\phi}'_{k}{\phi}^{\top}_k \textbf{C}^{-1}\mathbb{E}[(\delta_k -\mathbb{E}[\delta_k|{\theta}_k]){\phi}|{\theta}_k])\\
    & &-\mathbb{E}[(\delta_k -\mathbb{E}[\delta_k|{\theta}_k]) {\phi}_k|{\theta}_k] - \gamma\mathbb{E}[{\phi}' {\phi}^{\top}]\textbf{C}^{-1}\mathbb{E}[(\delta_k -\mathbb{E}[\delta_k|{\theta}_k]) {\phi}_k|{\theta}_k].
    \end{array}
\end{equation*}
It is easy to see that $Z_k$, $k\geq 0$ are integrable random variables and $\mathbb{E}[Z_{k+1}|\mathcal{G}(k)]=0$, $\forall k\geq0$. Further,
\begin{equation*}
\mathbb{E}[||Z_{k+1}||^2|\mathcal{G}(k)]\leq
c_3(1+||{\theta}_k||^2), k\geq 0
\end{equation*}
for some constant $c_3 \geq 0$, again because ${\phi}_k$, $r_k$, and ${\phi}_k'$ have
uniformly bounded second moments, it can be seen that for some constant.

Consider now the following ODE associated with (\ref{thetavmtdc}):
\begin{equation}
 \dot{{\theta}}(t) = (\textbf{I} - \mathbb{E}[\gamma {\phi}' {\phi}^{\top}]\textbf{C}^{-1})\mathbb{E}[(\delta -\mathbb{E}[\delta|{\theta}(t)]) {\phi}|{\theta}(t)].
    \label{thetavmtdcSlowerFinal}
\end{equation}
Let 
\begin{equation*}
\begin{array}{ccl}
 \vec{h}({\theta}(t))&=&(\textbf{I} - \mathbb{E}[\gamma {\phi}' {\phi}^{\top}]\textbf{C}^{-1})\mathbb{E}[(\delta -\mathbb{E}[\delta|{\theta}(t)]) {\phi}|{\theta}(t)]\\
    &=&(\textbf{C} - \mathbb{E}[\gamma {\phi}' {\phi}^{\top}])\textbf{C}^{-1}\mathbb{E}[(\delta -\mathbb{E}[\delta|{\theta}(t)]) {\phi}|{\theta}(t)]\\
    &=& (\mathbb{E}[{\phi} {\phi}^{\top}] - \mathbb{E}[\gamma {\phi}' {\phi}^{\top}])\textbf{C}^{-1}\mathbb{E}[(\delta -\mathbb{E}[\delta|{\theta}(t)]) {\phi}|{\theta}(t)]\\
    &=& \textbf{A}^{\top}\textbf{C}^{-1}(-\textbf{A}{\theta}(t)+{b}),
\end{array}
\end{equation*}
because $\mathbb{E}[(\delta -\mathbb{E}[\delta|{\theta}(t)]) {\phi}|{\theta}(t)]=-\textbf{A}{\theta}(t)+{b}$, where 
$\textbf{A} = \mathrm{Cov}({\phi},{\phi}-\gamma{\phi}')$, ${b}=\mathrm{Cov}(r,{\phi})$, and $\textbf{C}=\mathbb{E}[{\phi}{\phi}^{\top}]$

Therefore,
${\theta}^*=\textbf{A}^{-1}{b}$ can be seen to be the unique globally asymptotically
stable equilibrium for ODE (\ref{thetavmtdcSlowerFinal}).
Let $\vec{h}_{\infty}({\theta})=\lim_{r\rightarrow
\infty}\frac{\vec{h}(r{\theta})}{r}$. Then
$\vec{h}_{\infty}({\theta})=-\textbf{A}^{\top}\textbf{C}^{-1}\textbf{A}{\theta}$ is well-defined. 
Consider now the ODE
\begin{equation}
\dot{{\theta}}(t)=-\textbf{A}^{\top}\textbf{C}^{-1}\textbf{A}{\theta}(t).
\label{odethetavmtdcfinal}
\end{equation}

Because $\textbf{C}^{-1}$ is positive definite and $\textbf{A}$ has full rank (as it
is nonsingular by assumption), the matrix $\textbf{A}^{\top} \textbf{C}^{-1}\textbf{A}$ is also
positive definite. 
The ODE (\ref{odethetavmtdcfinal}) has the origin of its unique globally asymptotically stable equilibrium.
Thus, the assumption (A1) and (A2) are verified.

The proof is given above.
In the fastest time scale, the parameter $w$ converges to
$\mathbb{E}[\delta|{u}_k,{\theta}_k]$.
In the second fast time scale,
the parameter $u$ converges to $\textbf{C}^{-1}\mathbb{E}[(\delta-\mathbb{E}[\delta|{\theta}_k]){\phi}|{\theta}_k]$.
In the slower time scale,
the parameter ${\theta}$ converges to $\textbf{A}^{-1}{b}$.
\end{proof}

\subsection{Proof of Theorem 4}
\label{proofVMETD}
\begin{proof}
\label{th4proof}   
The proof of VMETD's convergence is also based on Borkar's Theorem   for
general stochastic approximation recursions with two time scales
\cite{borkar1997stochastic}. 

The  VMTD's solution is
${\theta}_{\text{VMETD}}=\textbf{A}_{\text{VMETD}}^{-1}{b}_{\text{VMETD}}$.
First, note that recursion (\ref{thetavmetd}) can be rewritten as
\begin{equation*}
 {\theta}_{k+1}\leftarrow {\theta}_k+\beta_k{\xi}(k),
\end{equation*}
 where
\begin{equation*}
 {\xi}(k)=\frac{\alpha_k}{\beta_k} (F_k \rho_k\delta_k - \omega_{k+1}){\phi}_k
\end{equation*}
 Due to the settings of step-size schedule $\alpha_k = o(\beta_k)$,
${\xi}(k)\rightarrow 0$ almost surely as $k\rightarrow\infty$. 
 That is the increments in iteration (\ref{omegavmetd}) are uniformly larger than
 those in (\ref{thetavmetd}), thus (\ref{omegavmetd}) is the faster recursion.
 Along the faster time scale, iterations of (\ref{thetavmetd}) and (\ref{omegavmetd})
 are associated with the ODEs system as follows:
\begin{equation}
 \dot{{\theta}}(t) = 0,
\label{vmetdthetaFast}
\end{equation}
\begin{equation}
 \dot{\omega}(t)=\mathbb{E}_{\mu}[F_t\rho_t\delta_t|{\theta}(t)]-\omega(t).
\label{vmetdomegaFast}
\end{equation}
 Based on the ODE (\ref{vmetdthetaFast}), ${\theta}(t)\equiv {\theta}$ when
 viewed from the faster timescale. 
 By the Hirsch lemma \cite{hirsch1989convergent}, it follows that
$||{\theta}_k-{\theta}||\rightarrow 0$ a.s. as $k\rightarrow \infty$ for some
${\theta}$ that depends on the initial condition ${\theta}_0$ of recursion
(\ref{thetavmetd}).
 Thus, the ODE pair (\ref{vmetdthetaFast})-(\ref{vmetdomegaFast}) can be written as
\begin{equation}
 \dot{\omega}(t)=\mathbb{E}_{\mu}[F_t\rho_t\delta_t|{\theta}]-\omega(t).
\label{vmetdomegaFastFinal}
\end{equation}
 Consider the function $h(\omega)=\mathbb{E}_{\mu}[F\rho\delta|{\theta}]-\omega$,
 i.e., the driving vector field of the ODE (\ref{vmetdomegaFastFinal}).
 It is easy to find that the function $h$ is Lipschitz with coefficient
$-1$.
 Let $h_{\infty}(\cdot)$ be the function defined by
 $h_{\infty}(\omega)=\lim_{x\rightarrow \infty}\frac{h(x\omega)}{x}$.
 Then  $h_{\infty}(\omega)= -\omega$,  is well-defined. 
 For (\ref{vmetdomegaFastFinal}), $\omega^*=\mathbb{E}_{\mu}[F\rho\delta|{\theta}]$
 is the unique globally asymptotically stable equilibrium.
 For the ODE
  \begin{equation}
 \dot{\omega}(t) = h_{\infty}(\omega(t))= -\omega(t),
 \label{vmetdomegaInfty}
 \end{equation}
 apply $\vec{V}(\omega)=(-\omega)^{\top}(-\omega)/2$ as its
 associated strict Liapunov function. Then,
 the origin of (\ref{vmetdomegaInfty}) is a globally asymptotically stable
 equilibrium.

 Consider now the recursion (\ref{omegavmetd}).
 Let
$M_{k+1}=(F_k\rho_k\delta_k-\omega_k)
 -\mathbb{E}_{\mu}[(F_k\rho_k\delta_k-\omega_k)|\mathcal{F}(k)]$,
 where $\mathcal{F}(k)=\sigma(\omega_l,{\theta}_l,l\leq k;{\phi}_s,{\phi}_s',r_s,s<k)$, 
$k\geq 1$ are the sigma fields
 generated by $\omega_0,{\theta}_0,\omega_{l+1},{\theta}_{l+1},{\phi}_l,{\phi}_l'$,
$0\leq l<k$.
 It is easy to verify that $M_{k+1},k\geq0$ are integrable random variables that 
 satisfy $\mathbb{E}[M_{k+1}|\mathcal{F}(k)]=0$, $\forall k\geq0$.
 Because ${\phi}_k$, $r_k$, and ${\phi}_k'$ have
 uniformly bounded second moments, it can be seen that for some constant
$c_1>0$, $\forall k\geq0$,
\begin{equation*}
 \mathbb{E}[||M_{k+1}||^2|\mathcal{F}(k)]\leq
 c_1(1+||\omega_k||^2+||{\theta}_k||^2).
\end{equation*}

 Now Assumptions (A1) and (A2) of \cite{borkar2000ode} are verified.
 Furthermore, Assumptions (TS) of \cite{borkar2000ode} are satisfied by our
 conditions on the step-size sequences $\alpha_k$, $\beta_k$. Thus,
 by Theorem 2.2 of \cite{borkar2000ode} we obtain that
$||\omega_k-\omega^*||\rightarrow 0$ almost surely as $k\rightarrow \infty$.

 Consider now the slower time scale recursion (\ref{thetavmetd}).
 Based on the above analysis, (\ref{thetavmetd}) can be rewritten as 

\begin{equation*}
    \begin{split}
 {\theta}_{k+1}&\leftarrow {\theta}_k+\alpha_k (F_k \rho_k\delta_k - \omega_k){\phi}_k -\alpha_k \omega_{k+1}{\phi}_k\\
&={\theta}_{k}+\alpha_k(F_k\rho_k\delta_k-\mathbb{E}_{\mu}[F_k\rho_k\delta_k|{\theta}_k]){\phi}_k\\
    &={\theta}_k+\alpha_k F_k \rho_k (R_{k+1}+\gamma {\theta}_k^{\top}{\phi}_{k+1}-{\theta}_k^{\top}{\phi}_k){\phi}_k -\alpha_k \mathbb{E}_{\mu}[F_k \rho_k \delta_k]{\phi}_k\\
    &= {\theta}_k+\alpha_k \{\underbrace{(F_k\rho_kR_{k+1}-\mathbb{E}_{\mu}[F_k\rho_k R_{k+1}]){\phi}_k}_{{b}_{\text{VMETD},k}}
 -\underbrace{(F_k\rho_k{\phi}_k({\phi}_k-\gamma{\phi}_{k+1})^{\top}-{\phi}_k\mathbb{E}_{\mu}[F_k\rho_k ({\phi}_k-\gamma{\phi}_{k+1})]^{\top})}_{\textbf{A}_{\text{VMETD},k}}{\theta}_k\}
\end{split}
\end{equation*}

 Let $\mathcal{G}(k)=\sigma({\theta}_l,l\leq k;{\phi}_s,{\phi}_s',r_s,s<k)$, 
$k\geq 1$ be the sigma fields
 generated by ${\theta}_0,{\theta}_{l+1},{\phi}_l,{\phi}_l'$,
$0\leq l<k$.
 Let 
$
 Z_{k+1} = Y_{k}-\mathbb{E}[Y_{k}|\mathcal{G}(k)],
$
 where
\begin{equation*}
 Y_{k}=(F_k\rho_k\delta_k-\mathbb{E}_{\mu}[F_k\rho_k\delta_k|{\theta}_k]){\phi}_k.
\end{equation*}
 Consequently,
\begin{equation*}
\begin{array}{ccl}
 \mathbb{E}_{\mu}[Y_k|\mathcal{G}(k)]&=&\mathbb{E}_{\mu}[(F_k\rho_k\delta_k-\mathbb{E}_{\mu}[F_k\rho_k\delta_k|{\theta}_k]){\phi}_k|\mathcal{G}(k)]\\
&=&\mathbb{E}_{\mu}[F_k\rho_k\delta_k{\phi}_k|{\theta}_k]
 -\mathbb{E}_{\mu}[\mathbb{E}_{\mu}[F_k\rho_k\delta_k|{\theta}_k]{\phi}_k]\\
&=&\mathbb{E}_{\mu}[F_k\rho_k\delta_k{\phi}_k|{\theta}_k]
 -\mathbb{E}_{\mu}[F_k\rho_k\delta_k|{\theta}_k]\mathbb{E}_{\mu}[{\phi}_k]\\
&=&\mathrm{Cov}(F_k\rho_k\delta_k|{\theta}_k,{\phi}_k),
\end{array}
\end{equation*}
 where $\mathrm{Cov}(\cdot,\cdot)$ is a covariance operator.

 Thus,
 \begin{equation*}
\begin{array}{ccl}
 Z_{k+1}&=&(F_k\rho_k\delta_k-\mathbb{E}[\delta_k|{\theta}_k]){\phi}_k-\mathrm{Cov}(F_k\rho_k\delta_k|{\theta}_k,{\phi}_k).
\end{array}
\end{equation*}
 It is easy to verify that $Z_{k+1},k\geq0$ are integrable random variables that 
 satisfy $\mathbb{E}[Z_{k+1}|\mathcal{G}(k)]=0$, $\forall k\geq0$.
 Also, because ${\phi}_k$, $r_k$, and ${\phi}_k'$ have
 uniformly bounded second moments, it can be seen that for some constant
$c_2>0$, $\forall k\geq0$,
\begin{equation*}
 \mathbb{E}[||Z_{k+1}||^2|\mathcal{G}(k)]\leq
 c_2(1+||{\theta}_k||^2).
\end{equation*}

 Consider now the following ODE associated with (\ref{thetavmetd}):
\begin{equation}
\begin{array}{ccl}
 \dot{{\theta}}(t)&=&-\textbf{A}_{\text{VMETD}}{\theta}(t)+{b}_{\text{VMETD}}.
\end{array}
\label{odethetavmetd}
\end{equation}
\begin{equation}
    \begin{split}
 \textbf{A}_{\text{VMETD}}&=\lim_{k \rightarrow \infty} \mathbb{E}[\textbf{A}_{\text{VMETD},k}]\\
&= \lim_{k \rightarrow \infty} \mathbb{E}_{\mu}[F_k \rho_k {\phi}_k ({\phi}_k - \gamma {\phi}_{k+1})^{\top}]- \lim_{k\rightarrow \infty} \mathbb{E}_{\mu}[  {\phi}_k]\mathbb{E}_{\mu}[F_k \rho_k ({\phi}_k - \gamma {\phi}_{k+1})]^{\top}\\  
% &= \lim_{k \rightarrow \infty} \mathbb{E}_{\mu}[\underbrace{{\phi}_k}_{X}\underbrace{F_k \rho_k  ({\phi}_k - \gamma {\phi}_{k+1})^{\top}}_{Y}]- \lim_{k\rightarrow \infty} \mathbb{E}_{\mu}[  {\phi}_k]\mathbb{E}_{\mu}[F_k \rho_k ({\phi}_k - \gamma {\phi}_{k+1})]^{\top}\\  
&= \lim_{k \rightarrow \infty} \mathbb{E}_{\mu}[{\phi}_kF_k \rho_k  ({\phi}_k - \gamma {\phi}_{k+1})^{\top}]- \lim_{k\rightarrow \infty} \mathbb{E}_{\mu}[  {\phi}_k]\mathbb{E}_{\mu}[F_k \rho_k ({\phi}_k - \gamma {\phi}_{k+1})]^{\top}\\ 
&= \lim_{k \rightarrow \infty} \mathbb{E}_{\mu}[{\phi}_kF_k \rho_k ({\phi}_k - \gamma {\phi}_{k+1})^{\top}]- \lim_{k \rightarrow \infty} \mathbb{E}_{\mu}[ {\phi}_k]\lim_{k \rightarrow \infty}\mathbb{E}_{\mu}[F_k \rho_k ({\phi}_k - \gamma {\phi}_{k+1})]^{\top}\\   
&=\sum_{s} f(s) {\phi}(s)({\phi}(s) - \gamma \sum_{s'}[\textbf{P}_{\pi}]_{ss'}{\phi}(s'))^{\top} - \sum_{s} d_{\mu}(s) {\phi}(s) * \sum_{s} f(s)({\phi}(s) - \gamma \sum_{s'}[\textbf{P}_{\pi}]_{ss'}{\phi}(s'))^{\top}  \\
&={{\Phi}}^{\top} \textbf{F} (\textbf{I} - \gamma \textbf{P}_{\pi}) {\Phi} - {{\Phi}}^{\top} \textbf{d}_{\mu} \textbf{f}^{\top} (\textbf{I} - \gamma \textbf{P}_{\mu}) {\Phi}  \\
&={{\Phi}}^{\top} (\textbf{F} - \textbf{d}_{\mu} \textbf{f}^{\top}) (\textbf{I} - \gamma \textbf{P}_{\pi}){{\Phi}} \\
&={{\Phi}}^{\top} (\textbf{F} (\textbf{I} - \gamma \textbf{P}_{\pi})-\textbf{d}_{\mu} \textbf{f}^{\top} (\textbf{I} - \gamma \textbf{P}_{\pi})){{\Phi}} \\
&={{\Phi}}^{\top} (\textbf{F} (\textbf{I} - \gamma \textbf{P}_{\pi})-\textbf{d}_{\mu} \textbf{d}_{\mu}^{\top} ){{\Phi}} \\
    \end{split}
\end{equation}
\begin{equation}
    \begin{split}
 {b}_{\text{VMETD}}&=\lim_{k \rightarrow \infty} \mathbb{E}[{b}_{\text{VMETD},k}]\\
&= \lim_{k \rightarrow \infty} \mathbb{E}_{\mu}[F_k\rho_kR_{k+1}{\phi}_k]- \lim_{k\rightarrow \infty} \mathbb{E}_{\mu}[{\phi}_k]\mathbb{E}_{\mu}[F_k\rho_kR_{k+1}]\\  
&= \lim_{k \rightarrow \infty} \mathbb{E}_{\mu}[{\phi}_kF_k\rho_kR_{k+1}]- \lim_{k\rightarrow \infty} \mathbb{E}_{\mu}[  {\phi}_k]\mathbb{E}_{\mu}[{\phi}_k]\mathbb{E}_{\mu}[F_k\rho_kR_{k+1}]\\ 
&= \lim_{k \rightarrow \infty} \mathbb{E}_{\mu}[{\phi}_kF_k\rho_kR_{k+1}]- \lim_{k \rightarrow \infty} \mathbb{E}_{\mu}[ {\phi}_k]\lim_{k \rightarrow \infty}\mathbb{E}_{\mu}[F_k\rho_kR_{k+1}]\\  
&=\sum_{s} f(s) {\phi}(s)r_{\pi} - \sum_{s} d_{\mu}(s) {\phi}(s) * \sum_{s} f(s)r_{\pi}  \\
&={{\Phi}}^{\top}(\textbf{F}-\textbf{d}_{\mu} \textbf{f}^{\top})\textbf{r}_{\pi} \\
    \end{split}
\end{equation}
 Let $\vec{h}({\theta}(t))$ be the driving vector field of the ODE
 (\ref{odethetavmetd}).
\begin{equation*}
 \vec{h}({\theta}(t))=-\textbf{A}_{\text{VMETD}}{\theta}(t)+{b}_{\text{VMETD}}.
\end{equation*}

 An ${\Phi}^{\top}{\text{X}}{\Phi}$ matrix of this
 form will be positive definite whenever the matrix ${\text{X}}$ is positive definite.
 Any matrix ${\text{X}}$ is positive definite if and only if
 the symmetric matrix ${\text{S}}={\text{X}}+{\text{X}}^{\top}$ is positive definite. 
 Any symmetric real matrix ${\text{S}}$ is positive definite if the absolute values of
 its diagonal entries are greater than the sum of the absolute values of the corresponding
 off-diagonal entries\cite{sutton2016emphatic}. 

\begin{equation}
    \label{rowsum}
    \begin{split}
 (\textbf{F} (\textbf{I} - \gamma \textbf{P}_{\pi})-\textbf{d}_{\mu} \textbf{d}_{\mu}^{\top} )\textbf{1}
    &=\textbf{F} (\textbf{I} - \gamma \textbf{P}_{\pi})\textbf{1}-\textbf{d}_{\mu} \textbf{d}_{\mu}^{\top} \textbf{1}\\
    &=\textbf{F}(\textbf{1}-\gamma \textbf{P}_{\pi} \textbf{1})-\textbf{d}_{\mu} \textbf{d}_{\mu}^{\top} \textbf{1}\\
    &=(1-\gamma)\textbf{F}\textbf{1}-\textbf{d}_{\mu} \textbf{d}_{\mu}^{\top} \textbf{1}\\
    &=(1-\gamma)\textbf{f}-\textbf{d}_{\mu} \textbf{d}_{\mu}^{\top} \textbf{1}\\
    &=(1-\gamma)\textbf{f}-\textbf{d}_{\mu} \\
    &=(1-\gamma)(\textbf{I}-\gamma\textbf{P}_{\pi}^{\top})^{-1}\textbf{d}_{\mu}-\textbf{d}_{\mu} \\
    &=(1-\gamma)[(\textbf{I}-\gamma\textbf{P}_{\pi}^{\top})^{-1}-\textbf{I}]\textbf{d}_{\mu} \\
    &=(1-\gamma)[\sum_{t=0}^{\infty}(\gamma\textbf{P}_{\pi}^{\top})^{t}-\textbf{I}]\textbf{d}_{\mu} \\
    &=(1-\gamma)[\sum_{t=1}^{\infty}(\gamma\textbf{P}_{\pi}^{\top})^{t}]\textbf{d}_{\mu} > 0 \\
    \end{split}
    \end{equation}
\begin{equation}
    \label{columnsum}
    \begin{split}
 \textbf{1}^{\top}(\textbf{F} (\textbf{I} - \gamma \textbf{P}_{\pi})-\textbf{d}_{\mu} \textbf{d}_{\mu}^{\top} )
    &=\textbf{1}^{\top}\textbf{F} (\textbf{I} - \gamma \textbf{P}_{\pi})-\textbf{1}^{\top}\textbf{d}_{\mu} \textbf{d}_{\mu}^{\top} \\
    &=\textbf{d}_{\mu}^{\top}-\textbf{1}^{\top}\textbf{d}_{\mu} \textbf{d}_{\mu}^{\top} \\
    &=\textbf{d}_{\mu}^{\top}- \textbf{d}_{\mu}^{\top} \\
    &=0
    \end{split}
\end{equation}
 (\ref{rowsum}) and (\ref{columnsum}) show that the matrix $\textbf{F} (\textbf{I} - \gamma \textbf{P}_{\pi})-\textbf{d}_{\mu} \textbf{d}_{\mu}^{\top}$ of
 diagonal entries are positive and its off-diagonal entries are negative. So each row sum plus the corresponding column sum is positive. 
 So $\textbf{A}_{\text{VMETD}}$ is positive definite.

 Therefore,
${\theta}^*=\textbf{A}_{\text{VMETD}}^{-1}{b}_{\text{VMETD}}$ can be seen to be the unique globally asymptotically
 stable equilibrium for ODE (\ref{odethetavmetd}).
 Let $\vec{h}_{\infty}({\theta})=\lim_{r\rightarrow
\infty}\frac{\vec{h}(r{\theta})}{r}$. Then
$\vec{h}_{\infty}({\theta})=-\textbf{A}_{\text{VMETD}}{\theta}$ is well-defined. 
 Consider now the ODE
\begin{equation}
 \dot{{\theta}}(t)=-\textbf{A}_{\text{VMETD}}{\theta}(t).
\label{odethetavmetdfinal}
\end{equation}
 The ODE (\ref{odethetavmetdfinal}) has the origin of its unique globally asymptotically stable equilibrium.
 Thus, the assumption (A1) and (A2) are verified.
    \end{proof}

\section{Experimental details}
\label{experimentaldetails}
The 2-state version of Baird's off-policy counterexample: All learning rates follow linear learning rate decay.
For TD algorithm, $\frac{\alpha_k}{\omega_k}=4$ and $\alpha_0 = 0.1$.
For TDC algorithm, $\frac{\alpha_k}{\zeta_k}=5$ and $\alpha_0 = 0.1$.
For VMTDC algorithm, $\frac{\alpha_k}{\zeta_k}=5$, $\frac{\alpha_k}{\omega_k}=4$,and $\alpha_0 = 0.1$.
For VMTD algorithm, $\frac{\alpha_k}{\omega_k}=4$ and $\alpha_0 = 0.1$.

The 2-state version of Baird's off-policy counterexample: All learning rates follow linear learning rate decay.
For TD algorithm, $\frac{\alpha_k}{\omega_k}=4$ and $\alpha_0 = 0.1$.
For TDC algorithm, $\frac{\alpha_k}{\zeta_k}=5$ and $\alpha_0 = 0.1$.For ETD algorithm, $\alpha_0 = 0.1$.
For VMTDC algorithm, $\frac{\alpha_k}{\zeta_k}=5$, $\frac{\alpha_k}{\omega_k}=4$,and $\alpha_0 = 0.1$.For VMETD algorithm, $\frac{\alpha_k}{\omega_k}=4$ and $\alpha_0 = 0.1$.
For VMTD algorithm, $\frac{\alpha_k}{\omega_k}=4$ and $\alpha_0 = 0.1$.

For all policy evaluation experiments, each experiment 
is independently run 100 times.

For the four control experiments: The learning rates for each 
algorithm in all experiments are shown in Table \ref{lrofways}.
For all control experiments, each experiment is independently run 50 times.

\begin{table*}[htb]
    \centering
    \caption{Learning rates ($lr$) of four control experiments.}
    \label{lrofways}
    \begin{tabular}{ccccc}
      \toprule
      \multicolumn{1}{c}{Algorithms ($lr$)} & Maze & Cliff walking & Mountain Car & Acrobot \\
      \midrule
      Sarsa($\alpha$) & 0.1 & 0.1 & 0.1 & 0.1 \\
      GQ($\alpha,\zeta$) & 0.1, 0.003 & 0.1, 0.004 & 0.1, 0.01 & 0.1, 0.01 \\
      EQ($\alpha$) & 0.006 & 0.005 & 0.001 & 0.0005 \\
      VMSarsa($\alpha,\beta$) & 0.1, 0.001 & 0.1, 1e-4 & 0.1, 1e-4 & 0.1, 1e-4 \\
      VMGQ($\alpha,\zeta,\beta$) & 0.1, 0.001, 0.001 & 0.1, 0.005, 1e-4 & 0.1, 5e-4, 1e-4 & 0.1, 5e-4, 1e-4 \\
      VMEQ($\alpha,\beta$) & 0.001, 0.0005 & 0.005, 0.0001 & 0.001, 0.0001 & 0.0005, 0.0001 \\
      Q-learning($\alpha$) & 0.1 & 0.1 & 0.1 & 0.1 \\
      VMQ($\alpha,\beta$) & 0.1, 0.001 & 0.1, 1e-4 & 0.1, 1e-4 & 0.1, 1e-4 \\
      \bottomrule
    \end{tabular}
  \end{table*}
\end{document}